\documentclass[11pt]{article}

\usepackage[a4paper,margin=1in]{geometry}
\usepackage{amsmath,amsthm,amssymb}
\usepackage[round,authoryear]{natbib}
\usepackage{graphicx}
\usepackage{booktabs}
\usepackage{microtype}
\usepackage{xcolor}
\usepackage{listings}
\usepackage{tikz}
\usepackage{pgfplots}
\usepackage{pgfplotstable}
\pgfplotsset{compat=1.18}
\usetikzlibrary{arrows.meta,positioning,shapes.geometric,fit}
\usepackage[colorlinks=true,linkcolor=black,citecolor=black,urlcolor=black]{hyperref}

\definecolor{cInk}{HTML}{1A1A1A}    
\definecolor{cMute}{HTML}{8A8A8A}   
\definecolor{cGold}{HTML}{B8860B}   
\definecolor{cGoldDeep}{HTML}{8C6508}
\definecolor{cRule}{HTML}{D9C28A}

\theoremstyle{plain}
\newtheorem{theorem}{Theorem}
\newtheorem{lemma}{Lemma}
\newtheorem{proposition}{Proposition}
\newtheorem{corollary}{Corollary}
\theoremstyle{definition}
\newtheorem{definition}{Definition}
\newtheorem{assumption}{Assumption}

\lstdefinestyle{prolog}{
  language=Prolog,
  basicstyle=\ttfamily\scriptsize\color{cInk},
  keywordstyle=\color{cGoldDeep},
  commentstyle=\itshape\color{cMute},
  morecomment=[l]{\%},
  showstringspaces=false,
  columns=fullflexible,
  breaklines=true,
  frame=single,
  rulecolor=\color{cRule},
  literate={\\+}{{\color{cGoldDeep}\textbackslash+}}2 {:-}{{\color{cGoldDeep}:-}}2
}

\pgfplotstableread[col sep=comma]{data/plot_curves_R1.csv}\curvesRone
\pgfplotstableread[col sep=comma]{data/plot_curves_R2.csv}\curvesRtwo
\pgfplotstableread[col sep=comma]{data/plot_endpoint_R1_base.csv}\endpointRoneBase
\pgfplotstableread[col sep=comma]{data/plot_endpoint_R1_hero.csv}\endpointRoneHero
\pgfplotstableread[col sep=comma]{data/plot_endpoint_R2_base.csv}\endpointRtwoBase
\pgfplotstableread[col sep=comma]{data/plot_endpoint_R2_hero.csv}\endpointRtwoHero
\pgfplotstableread[col sep=comma]{data/plot_expansion_R2.csv}\expansionRtwo

\title{From Black Box to Executable Logic: Explainable Reinforcement
Learning through Prolog Expert Systems}
\author{Eduardo C. Garrido-Merch\'an \\
  Institute of Research in Technology (IIT), Universidad Pontificia Comillas \\
  DIGNUM Center for AI Ethics and Data Rights \\
  Madrid, Spain \\
  \texttt{ecgarrido@comillas.edu}}
\date{}

\begin{document}
\maketitle

\begin{abstract}
A trained deep reinforcement learning policy is a black box, and we ask whether
it can be rewritten as an executable logic program that reproduces its
behaviour, that a person can read, a logic engine can run, and an optimizer can
edit. We present a three-stage post-hoc transformation that extracts a frozen
proximal policy optimization teacher, induces an ordered list of first-order
clauses from its decisions, and emits the result as a Prolog program executed by
an off-the-shelf logic engine; a subsequent expansion stage edits the rule base
and accepts an edit only when policy evaluation certifies a return increase. The
method rests on guarantees we prove and check rather than assert. In a finite
Markov decision process a performance-difference identity localizes the return
the student forfeits to the advantage of the teacher's action at the states
where the two disagree, weighted by the teacher's visitation, which turns a
worst-case return-loss bound that is vacuous at $\gamma = 0.99$ into an
advantage-gap certificate that is computed exactly and, on our task, tighter by a
median factor of about $13{,}700$ and non-vacuous in every one of fifteen seeds;
the expansion loop in addition improves the return monotonically and halts. On a
two-room key-and-door task with $16{,}944$ reachable states the expanded Prolog
program attains exact optimal return in every seed and, in a budget-capped
regime, exceeds the stochastic teacher on exact return in fifteen of fifteen
seeds. On a MiniGrid DoorKey task a first-order Prolog student distilled on the
$8 \times 8$ layout transfers to unseen $6 \times 6$ and $16 \times 16$ grids at
interquartile-mean return $0.947$ and $0.988$ with nine clauses, whereas decision
trees over absolute cell coordinates collapse to zero return off the training
size; the relational representation, not the choice between a tree and a logic
program, is what transfers. On propositional continuous control the same pipeline
substitutes the network in fidelity and beats the closest logic-based prior art,
but a first-order clause list has no relational structure to exploit in a four to
six dimensional observation, and it loses to CART and VIPER on return, on
CartPole in zero of fifteen seeds with a Holm-corrected $p = 3.7 \times 10^{-4}$;
we report this as the expected ceiling of the propositional instantiation rather
than conceal it. For that setting we also prove that the propositional conversion
is arbitrarily faithful as the resolution $B$ grows, with disagreement and return
gap closing at rate $O(1/B)$, and that the cost is exponential in the observation
dimension, $\Omega(B^{d-1})$ rules for an oblique decision boundary.
\end{abstract}

\section{Introduction}

A trained deep reinforcement learning policy is a dense numerical object, and
the practical question of what it has learned is usually answered indirectly,
through saliency maps, surrogate scores, or counterfactual probes, none of
which yields an artefact one can read, execute, and edit. The oldest form of
artificial intelligence offers a different kind of answer. A symbolic expert
system states its behaviour as a set of rules that a human can inspect and a
logic engine can run, and recent work has shown that the two traditions can be
brought into contact, with large language models used to draft the rule bases
of classical expert systems \citep{garrido2025gofai}. This paper takes the
converse direction inside reinforcement learning: rather than write the rules
by hand or by prompting, we recover them from a trained neural policy, and we
ask how far the recovered symbolic program can be pushed. The object we produce
is not a heat map or an approximate score but a Prolog program that plays the
task, and the question we answer is whether such a program can be made to match
or surpass the network it was distilled from.

Three properties distinguish the artefact we target. First, the extraction is
post-hoc: the teacher is any frozen policy, here a proximal policy optimization
agent \citep{schulman2017proximal}, and nothing in its training is modified to
make it interpretable. Second, the student is an executable first-order logic
program, an ordered list of clauses over a predicate vocabulary that runs
unchanged under SWI-Prolog \citep{wielemaker2012swi}, rather than a
propositional decision tree or a fitted scoring function. Third, the pipeline
does not stop at imitation. Once the rule base exists, a return-maximizing
expansion stage edits it and keeps only the edits that a policy-evaluation
oracle certifies to raise the return, so the symbolic student can move away
from the teacher wherever the teacher was wrong.

We make three contributions, organized around where a first-order symbolic
student earns its keep and where it does not. The first is the transformation
pipeline itself, comprising an extraction stage that queries the frozen
teacher's greedy action over a reachable state census or, when no census
exists, over a DAgger loop \citep{ross2011reduction}, an induction stage that
greedily grows an ordered list of first-order clauses in the manner of the FOIL
algorithm \citep{quinlan1990learning}, and an expansion stage that hill climbs
on the exact return over the space of rule-base edits, accepting an edit only
when a policy-evaluation oracle certifies an increase. The second is a
theoretical account whose centre is a machine-checkable certificate. We prove a
return-loss bound whose worst-case form, at most
$2 R_{\max} \epsilon^{\dagger} / (1-\gamma)^2$ with $\epsilon^{\dagger}$ the
disagreement rate weighted by the teacher's discounted visitation, is vacuous at
the discount factor we use, and we then sharpen it: a performance-difference
identity localizes the forfeited return to the advantage of the teacher's action
at the disagreeing states, which yields an advantage-gap certificate that removes
one factor of the effective horizon, coincides with the exact return gap
whenever the student weakly dominates the teacher on those states, and is
computed in closed form in a finite Markov decision process. We prove in
addition that the expansion loop improves the return monotonically and halts,
and, for the continuous-observation setting, that the propositional threshold
instantiation converts the network to arbitrary fidelity as the resolution $B$
grows, with disagreement and return gap closing at rate $O(1/B)$, subject to a
curse of dimensionality that we make precise as an $\Omega(B^{d-1})$ lower bound
on the number of rules an oblique decision boundary forces on any axis-aligned
list. The third contribution is an experimental study across two regimes that
separate the relational payoff from its absence. On a two-room key-and-door
environment small enough to admit exact policy evaluation, the expanded Prolog
program reaches exact optimal return in every seed of both a converged and a
budget-capped teacher regime, and in the capped regime exceeds the stochastic
teacher on exact return in fifteen of fifteen seeds by a mean paired margin of
$0.100$ with a $95\%$ confidence interval of $[0.096, 0.104]$, while the refined
certificate is non-vacuous in every seed and tighter than the worst-case bound
by a median factor of about $13{,}700$. On a MiniGrid DoorKey task the
first-order student distilled on the $8 \times 8$ layout transfers to unseen
$6 \times 6$ and $16 \times 16$ grids with nine clauses, at a return that
tracks the teacher, whereas decision trees over absolute cell coordinates, which
match on the training size, collapse to zero return at every other size, so the
relational representation is what carries the policy across scale. On four to six
dimensional continuous control, by contrast, the propositional student
substitutes the network in fidelity and beats the closest logic-based prior art
\citep{coppens2021rule} but loses on return to CART \citep{breiman1984classification}
and to VIPER \citep{bastani2018verifiable}, and we frame this loss as the
expected consequence of a representation with no relational structure to exploit
where the observation carries none. We report throughout the interquartile mean
with stratified bootstrap intervals in the manner of \citet{agarwal2021deep} and
correct multiple comparisons explicitly, and we are equally explicit about the
negative results, chief among them the CartPole return deficit in which the
first-order list is beaten in zero of fifteen seeds.

The paper proceeds as follows. We first position the method within the
neurosymbolic and explainable reinforcement learning landscape and then fix the
formal setting, the student class, and the three-stage algorithm. We prove the
return-loss bound, the advantage-gap refinement, the expansion guarantee, and
the resolution theory, and we then present the experiments in one section that
moves from the exact and certified key-and-door regime to the relational
DoorKey transfer and finally to the propositional continuous-control
instantiation with its honest baseline comparison. We close with the
limitations the guarantees do not cover and the generative continuation the
expansion oracle makes possible.

\section{Related work}

The work that turns a neural policy into an inspectable object organizes along
three axes, and naming them locates our contribution precisely. The first axis
separates post-hoc methods, which leave a fixed teacher untouched and recover a
surrogate from its decisions, from trained methods, which build interpretability
into the policy during learning. The second axis separates the target
representation, on one side the propositional surrogates whose tests are
thresholds on raw features, decision trees foremost among them
\citep{breiman1984classification}, and on the other the first-order or
relational surrogates whose clauses quantify over objects and relations. The
third axis separates a surrogate that merely scores or explains from one that is
itself executable and can be run as the policy. Interpretability here is a proxy
for ends such as auditability and contestability rather than a measured quantity
\citep{doshivelez2017rigorous}, and we take the operational stance that an
executable program one can read, run, and edit is the strongest such proxy an
extraction method can produce. Our method occupies the post-hoc, first-order,
executable cell, and adds two properties the axes do not name: the student is
expandable past its teacher under a return oracle, and its return loss is
certified. We now place the closest prior work against these axes.

The distillation of a trained policy into an interpretable surrogate is by now
a mature line. VIPER extracts a decision tree from a deep Q-network with a
DAgger-style procedure and proves a return bound for the tree student
\citep{bastani2018verifiable}; programmatically interpretable reinforcement
learning searches a domain-specific program space guided by a neural oracle
\citep{verma2018programmatically}, and its imitation-projected successor casts
the search as mirror descent \citep{verma2019imitation}. Linear model U-trees
fit piecewise-linear trees to a network's behaviour \citep{liu2018toward}, and
more recent work distils policies into editable programmatic trees and studies
when they can be corrected by hand \citep{kohler2024interpretable}. Our
extraction and induction stages sit squarely in this tradition, and Theorem
\ref{thm:distillation} is the decision-list analogue of the VIPER tree bound,
with the difference that in our finite setting the error rate is computed
exactly rather than estimated. The classical reduction of imitation to
no-regret online learning \citep{ross2011reduction} is the tool that would
replace our exhaustive census once the state space is too large to enumerate.

The target language separates our work from most of this line. A parallel body
of research learns first-order logic policies directly, either by training a
differentiable inductive logic layer \citep{jiang2019neural}, by guiding
symbolic rule search with a pretrained neural agent
\citep{delfosse2023interpretable}, by inventing explanatory predicates in
games \citep{sha2024expil}, or by blending symbolic and neural policies in a
single trainable model \citep{shindo2025blendrl}; concept-bottleneck agents
align a policy to object-centric concepts for the same interpretability end
\citep{delfosse2024interpretable}. These methods train the logic, whereas we
extract it post-hoc from a fixed neural teacher. The closest post-hoc work
synthesises set-valued propositional rules from a reinforcement learning policy
and refines them in the environment \citep{coppens2021rule}; our student is a
first-order, executable Prolog program rather than a propositional rule set,
and our expansion stage optimizes exact return rather than repairing fidelity.
The observation that symbolic policies can outperform deep networks has been
made for symbolic regression of closed-form controllers
\citep{landajuela2021discovering}, and finite-state controllers have been
extracted from recurrent policies \citep{koul2019learning}, but neither closes
the loop in which an extracted logic program is expanded past its own teacher.
Relational reinforcement learning \citep{dzeroski2001relational} is the
intellectual ancestor of the first-order representation we use, and the recent
survey of explainable reinforcement learning \citep{milani2024explainable}
places the entire distillation family in context and leaves the post-hoc
first-order executable cell that our method occupies essentially open.

\section{Problem setting}
\label{sec:setting}


We work in the infinite-horizon discounted setting with a finite state and
action space. All results in this paper rest on the following assumption,
which we reference by number throughout.

\begin{assumption}[Finite discounted MDP with bounded reward]
\label{asm:mdp}
The environment is a Markov decision process
$M = (\mathcal{S}, \mathcal{A}, P, r, \gamma, \mu_0)$ in which the state
space $\mathcal{S}$ and the action space $\mathcal{A}$ are finite,
$P(\cdot \mid s, a)$ is a probability distribution on $\mathcal{S}$ for every
pair $(s, a) \in \mathcal{S} \times \mathcal{A}$, the reward function
$r \colon \mathcal{S} \times \mathcal{A} \to \mathbb{R}$ satisfies
$|r(s, a)| \le R_{\max}$ for a known constant $R_{\max} < \infty$, the
discount factor satisfies $\gamma \in (0, 1)$, and $\mu_0$ is a probability
distribution on $\mathcal{S}$ from which the initial state is drawn.
\end{assumption}

A \emph{stationary policy} is a map $\pi \colon \mathcal{S} \to
\Delta(\mathcal{A})$, where $\Delta(\mathcal{A})$ denotes the set of
probability distributions on $\mathcal{A}$; we write $\pi(a \mid s)$ for the
probability that $\pi$ selects action $a$ in state $s$. A \emph{deterministic
policy} is a map $\pi \colon \mathcal{S} \to \mathcal{A}$, identified with
the stationary policy that puts all mass on $\pi(s)$. Given a policy $\pi$,
the process $(s_t, a_t)_{t \ge 0}$ with $s_0 \sim \mu_0$,
$a_t \sim \pi(\cdot \mid s_t)$, and $s_{t+1} \sim P(\cdot \mid s_t, a_t)$ is
the trajectory distribution induced by $\pi$, and we write
$\Pr^{\pi}_{\mu_0}$ and $\mathbb{E}^{\pi}_{\mu_0}$ for probabilities and
expectations under it.

\begin{definition}[Value functions and return]
\label{def:values}
Under Assumption~\ref{asm:mdp}, the \emph{state value function} of a policy
$\pi$ is
\[
V^{\pi}(s)
= \mathbb{E}^{\pi}\Biggl[\, \sum_{t=0}^{\infty} \gamma^{t} r(s_t, a_t)
\,\Bigm|\, s_0 = s \Biggr],
\]
the \emph{action value function} is
\[
Q^{\pi}(s, a)
= r(s, a) + \gamma \sum_{s' \in \mathcal{S}} P(s' \mid s, a)\, V^{\pi}(s'),
\]
and the \emph{advantage function} is
$A^{\pi}(s, a) = Q^{\pi}(s, a) - V^{\pi}(s)$. The \emph{return} of $\pi$ is
$J(\pi) = \mathbb{E}_{s \sim \mu_0}[V^{\pi}(s)]$.
\end{definition}

The series defining $V^{\pi}$ converges absolutely because
$|r(s_t, a_t)| \le R_{\max}$ and $\gamma \in (0, 1)$, and the geometric-series
bound yields
\begin{equation}
\label{eq:value-range}
|V^{\pi}(s)| \le \frac{R_{\max}}{1 - \gamma},
\qquad
|Q^{\pi}(s, a)| \le \frac{R_{\max}}{1 - \gamma},
\qquad
|A^{\pi}(s, a)| \le \frac{2 R_{\max}}{1 - \gamma}
\end{equation}
for every $\pi$, $s$, and $a$. The bound on $A^{\pi}$ follows from the
triangle inequality applied to the definition
$A^{\pi} = Q^{\pi} - V^{\pi}$. The ranges in \eqref{eq:value-range} supply
the constants in Theorem~\ref{thm:distillation} and
Proposition~\ref{prop:expansion}.

\begin{definition}[Teacher and argmax teacher]
\label{def:teacher}
The \emph{teacher} $\pi_T$ is a fixed stochastic stationary policy, in our
experiments the softmax policy of a converged or budget-capped PPO agent
\citep{schulman2017proximal}, frozen after training. Fix a total order on
$\mathcal{A}$. The \emph{argmax action} of the teacher at state $s$ is
\[
a_T(s) = \min \Bigl\{ a \in \mathcal{A} : \pi_T(a \mid s)
= \max_{a' \in \mathcal{A}} \pi_T(a' \mid s) \Bigr\},
\]
the minimum taken with respect to the fixed order on $\mathcal{A}$, so that
ties are broken deterministically. The \emph{argmax teacher} is the
deterministic policy $\bar{\pi}_T(s) = a_T(s)$.
\end{definition}

Throughout the paper the imitation target of the distillation stage is the
argmax teacher $\bar{\pi}_T$, not the stochastic policy $\pi_T$. This
convention is stated once here and used consistently: the student is trained
to reproduce $a_T(s)$, the return-loss bound of
Theorem~\ref{thm:distillation} is stated against $J(\bar{\pi}_T)$, and the
gap $J(\pi_T) - J(\bar{\pi}_T)$ between the stochastic teacher and its
argmax counterpart is a separate, sign-indefinite quantity that we compute
exactly rather than bound (Corollary~\ref{cor:stochastic-teacher} and
Proposition~\ref{prop:certificate}).

We now define the student class. The student is an ordered list of
first-order clauses over a fixed predicate vocabulary, executed with
decision-list semantics and emitted as a Prolog program.

\begin{definition}[Predicate vocabulary and grounded literals]
\label{def:vocab}
A \emph{predicate vocabulary} is a finite set
$\mathcal{P} = \{p_1, \dots, p_n\}$ of predicate symbols, each $p_i$ with an
arity $k_i \ge 0$ and, for each argument position, a finite argument domain
(in our instantiation, object constants such as $\mathrm{key}$,
$\mathrm{door}$, $\mathrm{goal}$ and direction constants such as
$\mathrm{up}$, $\mathrm{down}$, $\mathrm{left}$, $\mathrm{right}$), together
with an \emph{interpretation map} that assigns to every state
$s \in \mathcal{S}$ and every tuple of arguments drawn from the argument
domains a truth value. An \emph{atom} is an expression
$p_i(u_1, \dots, u_{k_i})$ whose arguments $u_j$ are constants from the
corresponding domains or variables; a \emph{literal} is an atom or a negated
atom. A \emph{grounding} of a set of literals is a substitution of constants
for all of its variables; a grounding is \emph{satisfied} in state $s$ when
every positive literal evaluates to true and every negated literal evaluates
to false under the interpretation map at $s$.
\end{definition}

\begin{definition}[Ordered decision list and student class]
\label{def:dlist}
A \emph{clause} is a pair $c = (B, \alpha)$ in which the body $B$ is a
finite conjunction of literals over the vocabulary $\mathcal{P}$ and the
head $\alpha$ is either a constant action in $\mathcal{A}$ or an action term
containing variables that also appear in $B$. An \emph{ordered decision
list} is a finite sequence $L = (c_1, \dots, c_m, c_{\mathrm{def}})$ of
clauses whose last element $c_{\mathrm{def}}$, the \emph{default clause},
has an empty body and a constant action head. The list $L$ defines a policy
$\pi_L$ as follows. Fix a total order on the finite set of groundings of
each clause body (in the executable program, the order in which SLD
resolution enumerates solutions). At state $s$, let $c_j$ be the first
clause in the list whose body admits a satisfied grounding at $s$, let
$\sigma$ be the first satisfied grounding of its body, and set $\pi_L(s)$ to
the action obtained by applying $\sigma$ to the head of $c_j$. The
\emph{student class} $\Pi_{\mathcal{L}}$ is the set of all policies $\pi_L$
representable by ordered decision lists over $\mathcal{P}$.
\end{definition}

Because the default clause has an empty body, it admits the empty grounding
at every state, so some clause fires at every $s$; because the first firing
clause and its first satisfied grounding are determined by the fixed orders,
the firing clause selects a unique action. Consequently every
$\pi_L \in \Pi_{\mathcal{L}}$ is a total deterministic policy, and
$\Pi_{\mathcal{L}}$ is a subset of the set $\mathcal{A}^{\mathcal{S}}$ of
deterministic policies. The results below use only this property; the
first-order structure of $\Pi_{\mathcal{L}}$ matters for compactness and
interpretability, not for the validity of the bounds. In the algorithmic
sections the generic symbol $\pi_S$ denotes the deterministic student, and
in all instantiations $\pi_S \in \Pi_{\mathcal{L}}$.

\begin{definition}[Discounted occupancy measure]
\label{def:occupancy}
Under Assumption~\ref{asm:mdp}, the \emph{normalized discounted occupancy
measure} of a policy $\pi$ is the probability distribution on $\mathcal{S}$
given by
\[
d^{\pi}(s)
= (1 - \gamma) \sum_{t=0}^{\infty} \gamma^{t}\,
\Pr^{\pi}_{\mu_0}(s_t = s).
\]
\end{definition}

The prefactor $(1 - \gamma)$ normalizes the geometric series, so
$\sum_{s} d^{\pi}(s) = 1$ and expressions of the form
$\mathbb{E}_{s \sim d^{\pi}}[f(s)]$ are ordinary expectations. The
occupancy measure weights states by how often, in discounted terms, the
policy $\pi$ visits them from $\mu_0$; it is the correct measure under which
to count imitation errors, because the performance-difference identity of
Lemma~\ref{lem:pdl} integrates advantages precisely against it.

\begin{definition}[Disagreement rates]
\label{def:disagreement}
Let $\pi_S$ be a deterministic student policy. The \emph{disagreement rate
under the stochastic teacher} is
\[
\epsilon
= \Pr_{s \sim d^{\pi_T}}\bigl[\pi_S(s) \neq a_T(s)\bigr]
= \sum_{s \in \mathcal{S}} d^{\pi_T}(s)\,
\mathbf{1}\bigl[\pi_S(s) \neq a_T(s)\bigr],
\]
and the \emph{disagreement rate under the argmax teacher} is
\[
\epsilon^{\dagger}
= \Pr_{s \sim d^{\bar{\pi}_T}}\bigl[\pi_S(s) \neq a_T(s)\bigr]
= \sum_{s \in \mathcal{S}} d^{\bar{\pi}_T}(s)\,
\mathbf{1}\bigl[\pi_S(s) \neq a_T(s)\bigr],
\]
where $\mathbf{1}[\cdot]$ denotes the indicator function.
\end{definition}

Both rates measure the same event, disagreement with the argmax action
$a_T(s)$, under two different visitation measures. The rate
$\epsilon^{\dagger}$ is the quantity that enters the return-loss bound of
Theorem~\ref{thm:distillation}, because the bound compares $\pi_S$ with the
deterministic target $\bar{\pi}_T$ and the performance-difference identity
then integrates against $d^{\bar{\pi}_T}$. The rate $\epsilon$ is the
fidelity statistic most natural to report against the deployed stochastic
teacher. In the finite setting of Assumption~\ref{asm:mdp} with a known
model, both are computed exactly by the census
(Proposition~\ref{prop:certificate}), so no change-of-measure or
concentrability argument is needed to relate them: the experiments report
both numbers.

\section{Extraction, induction, and expansion}
\label{sec:algorithm}

The pipeline maps a frozen teacher policy $\pi_T$ to an executable logic
program in three stages, whose formal guarantees are the subject of Section
\ref{sec:theory} and whose interfaces to those guarantees are the visitation
measures and disagreement rates of Section \ref{sec:setting}. Figure
\ref{fig:pipeline} summarizes the flow of artefacts.

The extraction stage records what the teacher does. Under the model access of
Assumption \ref{asm:model} the reachable set $\mathcal{S}_{\mathrm{r}}$ is
enumerated once, and the teacher's greedy action $a_T(s)$ of Definition
\ref{def:teacher} is queried at every reachable state, producing a complete
census of the target policy $\bar{\pi}_T$. The same rollouts that estimate the
teacher's discounted visitation $d^{\pi_T}$ are retained, because the induction
stage weights states by visitation and the two disagreement rates of
Definition \ref{def:disagreement} are reported against both the census measure
and the visitation measure. The census is the finite-state instrument that
replaces the sampling of a DAgger loop \citep{ross2011reduction}: because every
reachable state is labelled, the induction stage never queries the teacher off
its own distribution, and the disagreement rate that enters the return-loss
bound is exact.

The induction stage grows an ordered decision list of the class in Definition
\ref{def:dlist}. Following the greedy set-covering strategy of FOIL
\citep{quinlan1990learning}, the algorithm repeatedly selects the action head
and the conjunction of literals that best separate the states still labelled by
the teacher, scoring a candidate clause by the product of its coverage and its
precision under the visitation weights, specializing a clause by adding up to
three literals from the predicate vocabulary, and admitting the shared
direction variable that lets a single clause such as
$\mathrm{act}(s, \mathrm{move}(D)) \leftarrow \mathrm{dir\_to}(s, \mathrm{key},
D)$ ground to whichever compass action points along the path to the key. The
list is closed with a default clause set to the teacher's most frequent
residual action, which by the argument following Definition \ref{def:dlist}
makes the induced policy total and deterministic. The result is written to disk
as a Prolog program: a perception layer defines the state predicates, a tabled
breadth-first search computes the path-aware direction predicate while
respecting walls and the locked door, and the decision list is compiled to
clauses with a first-match cut discipline so that the first applicable rule
fires exactly as the decision-list semantics prescribe. The program is
executable input to SWI-Prolog \citep{wielemaker2012swi}, and at run time we
assert on a sample of states that the Prolog engine and the reference
evaluator select identical actions, so the artefact we analyse and the artefact
we execute are the same policy.

The expansion stage improves the program. It hill climbs over rule-base edits,
adding a specialized clause at a chosen position, reordering clauses, or
pruning a clause other than the default, and it accepts an edit only when the
exact return of Definition \ref{def:values} increases by at least a margin
$\tau$. The exact return is available because the environment is a finite
Markov decision process with a known model, so the acceptance test solves one
linear system rather than averaging noisy rollouts, and Proposition
\ref{prop:expansion} shows the loop improves monotonically and terminates.
Because the acceptance criterion is return, not fidelity, an accepted edit can
move the student away from the teacher on states the teacher visits rarely, and
this is precisely the mechanism by which the symbolic student surpasses the
network that produced it.

Two of the three stages survive intact when the state space is too large to
enumerate, and only the interfaces to the guarantees change. When no reachable
census exists, the extraction stage replaces it with the DAgger reduction of
imitation to no-regret online learning \citep{ross2011reduction}: the current
student is rolled out, the states it visits are relabelled by the frozen
teacher, and the union of teacher-visited and student-visited states drives the
next induction, so the induced list is trained on the off-distribution states
its own errors produce rather than on the teacher's distribution alone. Exact
policy evaluation is then replaced by Monte-Carlo return over a fixed
episode-seed stream shared by the teacher and every student, which reintroduces
the confidence parameters that the finite-model solves eliminate but leaves the
induction and the executable emission unchanged. The predicate vocabulary is the
one part of the pipeline that is task-specific, and its expressiveness is what
decides whether the first-order structure of Definition~\ref{def:dlist} is
exercised or left idle. On the MiniGrid DoorKey task
\citep{chevalier2023minigrid} the vocabulary is genuinely relational: a base
layer of nullary perception predicates such as $\mathrm{facing\_key}$,
$\mathrm{facing\_door}$, and $\mathrm{carrying\_key}$ reads the agent-centric
observation, and a path-aware navigation predicate $\mathrm{nav}(s,
\mathrm{obj}, A)$ computes, by a tabled breadth-first search over the traversable
cells that respects walls and the locked door, the action $A$ that advances
along a shortest path to the object $\mathrm{obj} \in \{\mathrm{key},
\mathrm{door}, \mathrm{goal}\}$. Because $\mathrm{nav}$ quantifies over cells and
returns a direction rather than testing a coordinate, a clause such as
$\mathrm{act}(s, A) \leftarrow \mathrm{nav}(s, \mathrm{goal}, A)$ is agnostic to
the size of the grid, and the same list transfers unchanged from the layout it
was distilled on to layouts it has never seen. On continuous control the
vocabulary is by contrast propositional, a grid of threshold tests on the
observation features, and the first-order machinery collapses to an ordered list
of axis-aligned conjunctions; Section~\ref{sec:experiments} shows that this is
exactly where the representational advantage disappears.

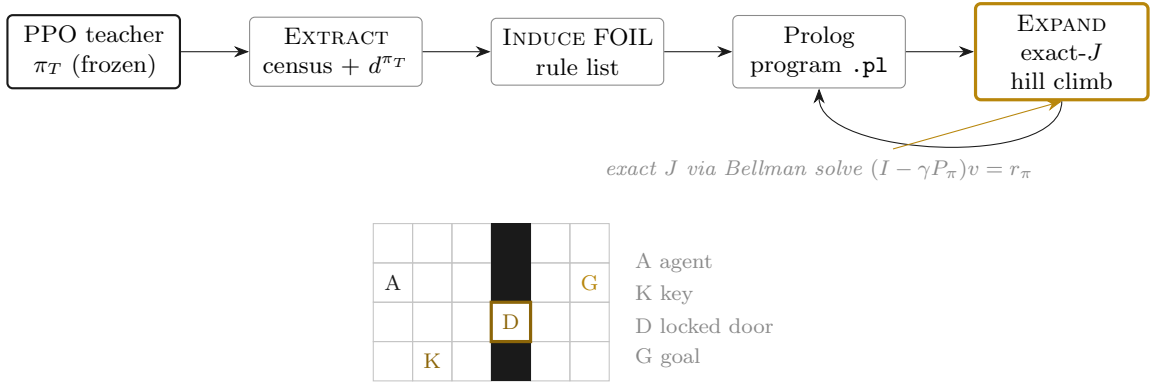
\begin{figure}[t]
\centering
\begin{tikzpicture}[
  node distance=6mm and 9mm,
  box/.style={draw=cInk, rounded corners=2pt, align=center, inner sep=4pt,
    font=\footnotesize, minimum height=8mm, text width=20mm},
  neural/.style={box, draw=cInk, thick},
  sym/.style={box, draw=cMute},
  hero/.style={box, draw=cGold, very thick},
  arr/.style={-{Stealth[length=2mm]}, cInk}]
\node[neural] (teacher) {PPO teacher $\pi_T$ (frozen)};
\node[sym, right=of teacher] (extract) {\textsc{Extract} census $+$ $d^{\pi_T}$};
\node[sym, right=of extract] (induce) {\textsc{Induce} FOIL rule list};
\node[sym, right=of induce] (prolog) {Prolog program \texttt{.pl}};
\node[hero, right=of prolog] (expand) {\textsc{Expand} exact-$J$ hill climb};
\draw[arr] (teacher) -- (extract);
\draw[arr] (extract) -- (induce);
\draw[arr] (induce) -- (prolog);
\draw[arr] (prolog) -- (expand);
\draw[arr] (expand.south) .. controls +(0,-9mm) and +(0,-9mm) .. (prolog.south);
\node[font=\scriptsize\itshape, color=cMute, below=8mm of prolog]
  (oracle) {exact $J$ via Bellman solve $(I-\gamma P_\pi)v = r_\pi$};
\draw[arr, cGold] (oracle) -- (expand.south);
\end{tikzpicture}

\vspace{4mm}

\begin{tikzpicture}[scale=0.52, every node/.style={font=\scriptsize}]
  \foreach \c in {1,...,6} {
    \foreach \r in {1,...,4} {
      \draw[cMute!50] (\c,\r) rectangle (\c+1,\r+1);
    }
  }
  \foreach \r in {1,3,4} {
    \fill[cInk] (4,\r) rectangle (5,\r+1);
  }
  \draw[cGoldDeep, very thick] (4,2) rectangle (5,3);
  \node[cGoldDeep] at (4.5,2.5) {D};
  \node[cInk] at (1.5,3.5) {A};
  \node[cGoldDeep] at (2.5,1.5) {K};
  \node[cGold] at (6.5,3.5) {G};
  \node[color=cMute, anchor=west] at (7.4,4) {A agent};
  \node[color=cMute, anchor=west] at (7.4,3.2) {K key};
  \node[color=cMute, anchor=west] at (7.4,2.4) {D locked door};
  \node[color=cMute, anchor=west] at (7.4,1.6) {G goal};
\end{tikzpicture}
\caption{Above, the transformation pipeline: a frozen PPO teacher is censused
and rolled out, a FOIL-style ordered rule list is induced and emitted as an
executable Prolog program, and an expansion loop edits the program, accepting
an edit only when an exact-return oracle certifies an increase. Below, the
KeyDoor environment: the agent must fetch the key in the left room and open the
locked door before it can reach the goal in the right room. Colours encode the
semantic roles used in every figure, with the teacher in ink, the distilled
student in grey, and the expanded student in gold.}
\label{fig:pipeline}
\end{figure}

\section{Theory}
\label{sec:theory}

We now state the guarantees. The first bounds the return the distilled program
forfeits relative to the greedy teacher. The second shows the expansion loop is
well behaved. The third turns the first into a certificate that is checked, not
merely asserted, on every run of the experiments. The fourth sharpens that
certificate: it localizes the forfeited return to the advantage of the teacher's
action at the disagreeing states, removes a factor of the effective horizon, and
becomes exactly non-vacuous when the student weakly dominates the teacher on
those states, which is the regime our distillation targets and the reason the
certificate acquires teeth on the key-and-door task where the worst-case bound
cannot. The fifth answers, for the continuous-observation setting of the later
experiments, the question of whether a neural policy can be converted into a
threshold-rule expert system at all: it establishes that the conversion is
arbitrarily faithful as the resolution grows, that the return gap then closes,
and that the cost of doing so is exponential in the observation dimension.


We now bound the return lost by replacing the argmax teacher with the
distilled student. The strategy is the standard one for imitation-style
guarantees, instantiated for a deterministic logic-program student and, in
our finite setting, with an exactly computable error rate rather than an
estimated one. The single technical ingredient is the performance-difference
lemma of \citet{kakade2002approximately}: the return gap between two
policies equals the discounted aggregate of the advantages of one policy's
actions measured under the other policy's value function and occupancy. On
states where the student reproduces the argmax action the advantage of the
teacher's action under the student's own value function vanishes, because
the advantage of a deterministic policy's own action is identically zero; on
the remaining states the advantage is bounded by the range
\eqref{eq:value-range}. The disagreement rate $\epsilon^{\dagger}$ of
Definition~\ref{def:disagreement} then converts the pointwise bound into the
final estimate. The only subtlety is the choice of imitation target: we
compare the student with the deterministic argmax teacher $\bar{\pi}_T$ of
Definition~\ref{def:teacher}, and we account for the stochastic teacher
separately in Corollary~\ref{cor:stochastic-teacher}, because the gap
$J(\pi_T) - J(\bar{\pi}_T)$ is sign-indefinite in general and is computed
exactly by the census of Proposition~\ref{prop:certificate}.

We first state and prove the performance-difference lemma in the form we
use. The result is due to \citet{kakade2002approximately}; we include the
short telescoping proof to keep the paper self-contained and because the
certificate of Proposition~\ref{prop:certificate} checks this identity
numerically.

\begin{lemma}[Performance-difference lemma, \citealp{kakade2002approximately}]
\label{lem:pdl}
Let Assumption~\ref{asm:mdp} hold, and let $\pi$ and $\pi'$ be stationary
policies. Then
\begin{equation}
\label{eq:pdl}
J(\pi) - J(\pi')
= \frac{1}{1 - \gamma}\,
\mathbb{E}_{s \sim d^{\pi}}\,
\mathbb{E}_{a \sim \pi(\cdot \mid s)}
\bigl[ A^{\pi'}(s, a) \bigr].
\end{equation}
\end{lemma}

\begin{proof}
Let $(s_t, a_t)_{t \ge 0}$ be the trajectory process induced by $\pi$ from
$s_0 \sim \mu_0$. Since $J(\pi') = \mathbb{E}_{s_0 \sim \mu_0}[V^{\pi'}(s_0)]$
by Definition~\ref{def:values}, we may write the gap as
\begin{align}
J(\pi) - J(\pi')
&= \mathbb{E}^{\pi}_{\mu_0}\Biggl[\,
\sum_{t=0}^{\infty} \gamma^{t} r(s_t, a_t) \Biggr]
- \mathbb{E}^{\pi}_{\mu_0}\bigl[ V^{\pi'}(s_0) \bigr] \notag \\
&= \mathbb{E}^{\pi}_{\mu_0}\Biggl[\,
\sum_{t=0}^{\infty} \gamma^{t}
\Bigl( r(s_t, a_t) + \gamma V^{\pi'}(s_{t+1}) - V^{\pi'}(s_t) \Bigr)
\Biggr]. \notag
\end{align}
The second equality holds because the added terms telescope: the partial sum
$\sum_{t=0}^{T} \gamma^{t}\bigl(\gamma V^{\pi'}(s_{t+1}) - V^{\pi'}(s_t)\bigr)$
equals $\gamma^{T+1} V^{\pi'}(s_{T+1}) - V^{\pi'}(s_0)$, and the boundary
term $\gamma^{T+1} V^{\pi'}(s_{T+1})$ vanishes as $T \to \infty$ by the
value bound \eqref{eq:value-range} and $\gamma < 1$; the interchange of
limit and expectation is justified by dominated convergence, the summands
being uniformly bounded by
$(R_{\max} + (1 + \gamma) R_{\max} / (1 - \gamma))\gamma^{t}$, a summable
envelope. Next, conditioning each term on $(s_t, a_t)$ and using the Markov
property,
\[
\mathbb{E}^{\pi}_{\mu_0}\bigl[ V^{\pi'}(s_{t+1}) \mid s_t, a_t \bigr]
= \sum_{s' \in \mathcal{S}} P(s' \mid s_t, a_t)\, V^{\pi'}(s'),
\]
so that, by the definition of $Q^{\pi'}$ in Definition~\ref{def:values},
\[
\mathbb{E}^{\pi}_{\mu_0}\Bigl[
r(s_t, a_t) + \gamma V^{\pi'}(s_{t+1}) - V^{\pi'}(s_t)
\Bigm| s_t, a_t \Bigr]
= Q^{\pi'}(s_t, a_t) - V^{\pi'}(s_t)
= A^{\pi'}(s_t, a_t).
\]
Substituting and exchanging summation over $t$ with the expectation, which
is again licensed by the summable envelope
$\gamma^{t}\, 2 R_{\max} / (1 - \gamma)$ on $|A^{\pi'}|$, we obtain
\begin{align}
J(\pi) - J(\pi')
&= \sum_{t=0}^{\infty} \gamma^{t}\,
\mathbb{E}^{\pi}_{\mu_0}\bigl[ A^{\pi'}(s_t, a_t) \bigr] \notag \\
&= \frac{1}{1 - \gamma} \sum_{s \in \mathcal{S}} d^{\pi}(s)
\sum_{a \in \mathcal{A}} \pi(a \mid s)\, A^{\pi'}(s, a), \notag
\end{align}
where the last step unpacks
$\mathbb{E}^{\pi}_{\mu_0}[A^{\pi'}(s_t, a_t)]
= \sum_{s} \Pr^{\pi}_{\mu_0}(s_t = s) \sum_{a} \pi(a \mid s) A^{\pi'}(s, a)$
and applies the definition of the occupancy measure
(Definition~\ref{def:occupancy}), whose factor $(1 - \gamma)$ produces the
prefactor $1 / (1 - \gamma)$. This is \eqref{eq:pdl}.
\end{proof}

\begin{theorem}[Return-loss bound for the distilled student]
\label{thm:distillation}
Let Assumption~\ref{asm:mdp} hold. Let $\bar{\pi}_T$ be the argmax teacher
of Definition~\ref{def:teacher}, let $\pi_S$ be any deterministic policy,
in particular any decision-list student
$\pi_S \in \Pi_{\mathcal{L}}$ of Definition~\ref{def:dlist}, and let
$\epsilon^{\dagger} = \Pr_{s \sim d^{\bar{\pi}_T}}[\pi_S(s) \neq a_T(s)]$ be
the disagreement rate of Definition~\ref{def:disagreement}. Then
\begin{equation}
\label{eq:distillation-bound}
J(\bar{\pi}_T) - J(\pi_S)
\le \frac{2 R_{\max}\, \epsilon^{\dagger}}{(1 - \gamma)^2},
\end{equation}
and the same right-hand side bounds
$\bigl| J(\bar{\pi}_T) - J(\pi_S) \bigr|$. In particular, if
$\epsilon^{\dagger} = 0$ then $J(\pi_S) = J(\bar{\pi}_T)$.
\end{theorem}

\begin{proof}
Apply Lemma~\ref{lem:pdl} with $\pi = \bar{\pi}_T$ and $\pi' = \pi_S$. Since
$\bar{\pi}_T$ is deterministic with $\bar{\pi}_T(s) = a_T(s)$, the inner
expectation over actions collapses to a point evaluation and \eqref{eq:pdl}
reads
\begin{equation}
\label{eq:pdl-instantiated}
J(\bar{\pi}_T) - J(\pi_S)
= \frac{1}{1 - \gamma} \sum_{s \in \mathcal{S}}
d^{\bar{\pi}_T}(s)\, A^{\pi_S}\bigl(s, a_T(s)\bigr).
\end{equation}
We split the sum according to whether the student agrees with the argmax
action at $s$. On the agreement set
$\{ s : \pi_S(s) = a_T(s) \}$ the summand vanishes, because for a
deterministic policy the advantage of its own action is zero:
\[
A^{\pi_S}\bigl(s, \pi_S(s)\bigr)
= Q^{\pi_S}\bigl(s, \pi_S(s)\bigr) - V^{\pi_S}(s)
= 0,
\]
where the second equality holds since $V^{\pi_S}(s) =
Q^{\pi_S}(s, \pi_S(s))$ for deterministic $\pi_S$, by
Definition~\ref{def:values}. On the disagreement set
$D = \{ s : \pi_S(s) \neq a_T(s) \}$ we invoke the advantage range
\eqref{eq:value-range}, which gives
$|A^{\pi_S}(s, a_T(s))| \le 2 R_{\max} / (1 - \gamma)$ pointwise. Combining
the two cases in \eqref{eq:pdl-instantiated},
\begin{align}
J(\bar{\pi}_T) - J(\pi_S)
&= \frac{1}{1 - \gamma} \sum_{s \in D}
d^{\bar{\pi}_T}(s)\, A^{\pi_S}\bigl(s, a_T(s)\bigr) \notag \\
&\le \frac{1}{1 - \gamma} \cdot \frac{2 R_{\max}}{1 - \gamma}
\sum_{s \in D} d^{\bar{\pi}_T}(s) \notag \\
&= \frac{2 R_{\max}\, \epsilon^{\dagger}}{(1 - \gamma)^2}, \notag
\end{align}
where the final equality is the definition of $\epsilon^{\dagger}$
(Definition~\ref{def:disagreement}). The same computation applied to the
absolute value of \eqref{eq:pdl-instantiated}, using the triangle
inequality before bounding $|A^{\pi_S}|$, yields
$|J(\bar{\pi}_T) - J(\pi_S)| \le 2 R_{\max} \epsilon^{\dagger} /
(1 - \gamma)^2$. Finally, if $\epsilon^{\dagger} = 0$ then
$d^{\bar{\pi}_T}(s) = 0$ for every $s \in D$, so every summand in
\eqref{eq:pdl-instantiated} vanishes and the gap is exactly zero.
\end{proof}

\begin{corollary}[Accounting for the stochastic teacher]
\label{cor:stochastic-teacher}
Under the conditions of Theorem~\ref{thm:distillation}, with
$\Delta_T = J(\pi_T) - J(\bar{\pi}_T)$ the argmax gap of the stochastic
teacher,
\[
J(\pi_T) - J(\pi_S)
\le \Delta_T + \frac{2 R_{\max}\, \epsilon^{\dagger}}{(1 - \gamma)^2}.
\]
\end{corollary}

\begin{proof}
Write $J(\pi_T) - J(\pi_S) = \bigl( J(\pi_T) - J(\bar{\pi}_T) \bigr) +
\bigl( J(\bar{\pi}_T) - J(\pi_S) \bigr)$ and bound the second summand by
Theorem~\ref{thm:distillation}.
\end{proof}

The quantity $\Delta_T$ carries no general sign: $a_T(s)$ maximizes the
teacher's action probability, not the teacher's action value, so the argmax
teacher can be better or worse than the stochastic teacher it is derived
from. We therefore do not bound $\Delta_T$; under
Assumption~\ref{asm:mdp} with a known model it is computed exactly by
solving two Bellman linear systems (Proposition~\ref{prop:certificate}), and
the experiments report it alongside both sides of
\eqref{eq:distillation-bound}.

Theorem~\ref{thm:distillation} says the following. If the student reproduces
the argmax teacher's action on all but an $\epsilon^{\dagger}$-fraction of
states, the fraction weighted by the argmax teacher's own discounted
visitation, then the student forfeits at most
$2 R_{\max} \epsilon^{\dagger} / (1 - \gamma)^2$ return, and a perfectly
faithful student forfeits nothing. The theorem does not say that the student
tracks the stochastic teacher $\pi_T$; that comparison acquires the
sign-indefinite offset $\Delta_T$ of
Corollary~\ref{cor:stochastic-teacher}. Nor does it say anything about
disagreements on states the argmax teacher rarely visits: a state with
$d^{\bar{\pi}_T}(s) \approx 0$ contributes nothing to $\epsilon^{\dagger}$,
which is precisely the blind spot the census of
Proposition~\ref{prop:certificate} exposes and the EXPAND loop of
Proposition~\ref{prop:expansion} exploits. The bound is the decision-list
analogue of the tree-student guarantee of VIPER
\citep{bastani2018verifiable}, with one material difference: in our finite
setting $\epsilon^{\dagger}$ is computed exactly rather than estimated from
rollouts, so \eqref{eq:distillation-bound} is a certificate rather than a
high-probability statement. The constant is the standard one for imitation
by action matching: the factor $2 R_{\max} / (1 - \gamma)$ is the worst-case
advantage of a single deviation, the additional $1 / (1 - \gamma)$ reflects
the compounding of deviations along the discounted horizon, and the
quadratic horizon dependence cannot be improved in general for
disagreement-based bounds. In the myopic limit $\gamma \to 0$ the bound
degenerates gracefully to the single-step statement
$J(\bar{\pi}_T) - J(\pi_S) \le 2 R_{\max} \epsilon^{\dagger}$, the largest
possible one-step reward difference incurred with probability
$\epsilon^{\dagger}$. The bound is informative only when
$\epsilon^{\dagger} = O\bigl((1 - \gamma)^2\bigr)$: with the KeyDoor
instantiation $\gamma = 0.99$ the prefactor is
$2 R_{\max} \cdot 10^{4}$, so the certificate has teeth only for the very
small exact disagreement rates the census measures, which is the regime the
experiments target.


The EXPAND stage performs hill climbing on the return over the decision-list
class. An \emph{edit operator} maps an ordered decision list to another
ordered decision list; the operators we use add a specialized clause at a
chosen position, reorder existing clauses, or prune a clause other than the
default clause, so every edit output remains a well-formed list in the sense
of Definition~\ref{def:dlist} and in particular remains total and
deterministic. For a list $L$, let $\mathcal{E}(L)$ denote the set of
candidate lists proposed for $L$, that is, the images of $L$ under the
admissible edit operators instantiated over the candidate pool. The loop
maintains a current list $L_k$, evaluates $J(\pi_{L'})$ for candidates
$L' \in \mathcal{E}(L_k)$, accepts a candidate if it clears the current
return by at least a margin $\tau$, and stops when no candidate does.

\begin{assumption}[Exact evaluation and finite proposals]
\label{asm:expand}
The EXPAND loop has access to an oracle that returns the exact return
$J(\pi_L)$ of Definition~\ref{def:values} for every queried list $L$, the
acceptance margin satisfies $\tau > 0$, the candidate set $\mathcal{E}(L)$
is finite for every list $L$, and an edit $L \to L'$ is accepted only if
$J(\pi_{L'}) \ge J(\pi_L) + \tau$.
\end{assumption}

Under Assumption~\ref{asm:mdp} the exact-oracle requirement is not an
idealization: Proposition~\ref{prop:certificate} shows that $J(\pi_L)$ is
obtained by solving one linear system, so acceptance decisions are free of
Monte Carlo noise. We call a list $L$ a \emph{$\tau$-local optimum} with
respect to $\mathcal{E}$ when $J(\pi_{L'}) < J(\pi_L) + \tau$ for every
$L' \in \mathcal{E}(L)$, that is, when no admissible edit improves the
return by the margin.

\begin{proposition}[Monotone improvement and finite termination of EXPAND]
\label{prop:expansion}
Let Assumptions~\ref{asm:mdp} and~\ref{asm:expand} hold, and let
$L_0, L_1, L_2, \dots$ be the sequence of lists produced by the EXPAND
loop, where $L_{k+1} \in \mathcal{E}(L_k)$ is an accepted candidate. First,
the returns increase monotonically:
$J(\pi_{L_{k+1}}) \ge J(\pi_{L_k}) + \tau$ for every accepted edit, so the
last iterate is also the best iterate. Second, the number $N$ of accepted
edits satisfies
\[
N \;\le\; \frac{J_{\max} - J(\pi_{L_0})}{\tau}
\;\le\; \frac{2 R_{\max}}{(1 - \gamma)\, \tau},
\]
where $J_{\max} = \max_{\pi} J(\pi) \le R_{\max} / (1 - \gamma)$ is the
optimal return, the maximum over all stationary policies being attained by
an optimal deterministic policy, which exists in a finite discounted MDP
under Assumption~\ref{asm:mdp}. Third, the loop terminates after finitely many oracle calls,
and its final list is a $\tau$-local optimum with respect to $\mathcal{E}$.
\end{proposition}

\begin{proof}
The first claim restates the acceptance rule of
Assumption~\ref{asm:expand}: an edit is accepted only if it increases the
exact return by at least $\tau$, and since the oracle is exact the recorded
improvement is the true improvement; monotonicity of the sequence
$(J(\pi_{L_k}))_k$ follows by induction on $k$, and the last iterate
maximizes $J$ over the iterates because each term dominates its
predecessors. For the second claim, chain the accepted improvements: after
$N$ accepted edits,
\begin{align}
J(\pi_{L_N})
&= J(\pi_{L_0}) + \sum_{k=0}^{N-1}
\bigl( J(\pi_{L_{k+1}}) - J(\pi_{L_k}) \bigr) \notag \\
&\ge J(\pi_{L_0}) + N \tau, \notag
\end{align}
where the inequality applies the acceptance margin to each summand. Since
$J(\pi_{L_N}) \le J_{\max}$ by definition of $J_{\max}$, rearranging gives
$N \le (J_{\max} - J(\pi_{L_0})) / \tau$. The coarser bound follows from
the value range \eqref{eq:value-range}, which gives
$J_{\max} \le R_{\max} / (1 - \gamma)$ and
$J(\pi_{L_0}) \ge -R_{\max} / (1 - \gamma)$, hence
$J_{\max} - J(\pi_{L_0}) \le 2 R_{\max} / (1 - \gamma)$. For the third
claim we observe that each iteration of the loop either accepts an edit or
exhausts the finite candidate set $\mathcal{E}(L_k)$ without acceptance and
stops. Acceptances occur at most $N \le 2 R_{\max} / ((1 - \gamma) \tau)$
times by the second claim, and every iteration issues at most
$|\mathcal{E}(L_k)| < \infty$ oracle calls by
Assumption~\ref{asm:expand}, so the total number of oracle calls is finite
and the loop halts. At the halting list $L_{\mathrm{fin}}$ no candidate
$L' \in \mathcal{E}(L_{\mathrm{fin}})$ satisfies
$J(\pi_{L'}) \ge J(\pi_{L_{\mathrm{fin}}}) + \tau$, which is the definition
of a $\tau$-local optimum.
\end{proof}

Proposition~\ref{prop:expansion} is a statement about the loop, not about
global optimality: the final list is unimprovable by single admissible
edits at margin $\tau$, and nothing more. The guarantee is a last-iterate
guarantee, and by monotonicity the last iterate coincides with the best
iterate, so no averaging over iterates is involved. Two consequences matter
for the paper's claims. First, EXPAND never degrades the distilled student:
combining the monotonicity with Theorem~\ref{thm:distillation} applied to
the initial list gives
$J(\pi_{L_{\mathrm{fin}}}) \ge J(\pi_{L_0}) \ge J(\bar{\pi}_T) -
2 R_{\max} \epsilon^{\dagger}_0 / (1 - \gamma)^2$, where
$\epsilon^{\dagger}_0$ is the disagreement rate of the initial list.
Second, nothing upper-bounds $J(\pi_{L_{\mathrm{fin}}})$ by
$J(\bar{\pi}_T)$ or $J(\pi_T)$: when the teacher is imperfect on states of
low teacher visitation, an accepted edit that changes actions only on such
states raises the student strictly above the teacher, and the acceptance
test certifies each such crossing exactly. The bound on $N$ depends only
on $R_{\max}$, $\gamma$, and $\tau$; it is loose in practice, since typical
runs accept far fewer edits than $2 R_{\max} / ((1 - \gamma) \tau)$, and it
becomes weak as $\tau \to 0$ or $\gamma \to 1$, the price of a guarantee
that holds for every edit sequence the heuristic proposer may generate.


Theorem~\ref{thm:distillation} relates three quantities, the two returns and
the disagreement rate. In a finite MDP with a known model all three are
exactly computable, so the bound is not merely a guarantee about unseen
quantities: both sides of \eqref{eq:distillation-bound} can be evaluated and
the inequality checked mechanically. This section records the computation
and its correctness. We need one assumption beyond
Assumption~\ref{asm:mdp}.

\begin{assumption}[Exact model access]
\label{asm:model}
The transition kernel $P$, the reward function $r$, the discount factor
$\gamma$, and the initial distribution $\mu_0$ of
Assumption~\ref{asm:mdp} are known, and the environment exposes a procedure
that enumerates the reachable state set
\[
\mathcal{S}_{\mathrm{r}}
= \bigcup_{k \ge 0} \mathcal{S}_k,
\qquad
\mathcal{S}_0 = \operatorname{supp}(\mu_0),
\qquad
\mathcal{S}_{k+1} = \mathcal{S}_k \cup
\bigl\{ s' : P(s' \mid s, a) > 0
\text{ for some } s \in \mathcal{S}_k,\, a \in \mathcal{A} \bigr\}.
\]
\end{assumption}

The set $\mathcal{S}_{\mathrm{r}}$ is finite because
$\mathcal{S}$ is finite, the increasing union stabilizes after at most
$|\mathcal{S}|$ steps, and $\mathcal{S}_{\mathrm{r}}$ is closed under every
transition of every policy: if $s \in \mathcal{S}_{\mathrm{r}}$ and
$P(s' \mid s, a) > 0$ for any $a$, then $s' \in \mathcal{S}_{\mathrm{r}}$
by construction. Consequently every trajectory from $\mu_0$ under any policy
remains in $\mathcal{S}_{\mathrm{r}}$ with probability one, every occupancy
measure of Definition~\ref{def:occupancy} is supported on
$\mathcal{S}_{\mathrm{r}}$, and all linear algebra below may be carried out
on vectors and matrices indexed by $\mathcal{S}_{\mathrm{r}}$ alone. We
write $n = |\mathcal{S}_{\mathrm{r}}|$ and treat functions on
$\mathcal{S}_{\mathrm{r}}$ as vectors in $\mathbb{R}^{n}$.

For a stationary policy $\pi$, define the policy-induced reward vector and
transition matrix
\[
r_{\pi}(s) = \sum_{a \in \mathcal{A}} \pi(a \mid s)\, r(s, a),
\qquad
P_{\pi}(s, s') = \sum_{a \in \mathcal{A}} \pi(a \mid s)\, P(s' \mid s, a),
\]
for $s, s' \in \mathcal{S}_{\mathrm{r}}$; for a deterministic policy these
reduce to $r_{\pi}(s) = r(s, \pi(s))$ and
$P_{\pi}(s, \cdot) = P(\cdot \mid s, \pi(s))$. The matrix $P_{\pi}$ is
row-stochastic, since each row is a mixture of the probability vectors
$P(\cdot \mid s, a)$ and $\mathcal{S}_{\mathrm{r}}$ is closed, so no mass
escapes the index set.

\begin{proposition}[Exact census and machine-checkable certificate]
\label{prop:certificate}
Let Assumptions~\ref{asm:mdp} and~\ref{asm:model} hold, and let $\pi$ be
any stationary policy. First, the matrix $I - \gamma P_{\pi}$ is
invertible, and the value vector
$v^{\pi} = (V^{\pi}(s))_{s \in \mathcal{S}_{\mathrm{r}}}$ is the unique
solution of the Bellman linear system
\begin{equation}
\label{eq:bellman-system}
v^{\pi} = r_{\pi} + \gamma P_{\pi} v^{\pi},
\qquad \text{that is,} \qquad
v^{\pi} = (I - \gamma P_{\pi})^{-1} r_{\pi},
\end{equation}
so that $J(\pi) = \mu_0^{\top} v^{\pi}$. Second, the occupancy measure of
Definition~\ref{def:occupancy} satisfies
\begin{equation}
\label{eq:occupancy-system}
d^{\pi} = (1 - \gamma) \bigl( I - \gamma P_{\pi}^{\top} \bigr)^{-1} \mu_0.
\end{equation}
Third, applying \eqref{eq:bellman-system} to $\pi_T$, $\bar{\pi}_T$, and
$\pi_S$, and \eqref{eq:occupancy-system} to $\pi_T$ and $\bar{\pi}_T$,
yields in finitely many arithmetic operations the exact values of
$J(\pi_T)$, $J(\bar{\pi}_T)$, $J(\pi_S)$, $\Delta_T$, $\epsilon$, and
$\epsilon^{\dagger}$, and therefore of both sides of
\eqref{eq:distillation-bound}; the certified inequality
\[
\frac{2 R_{\max}\, \epsilon^{\dagger}}{(1 - \gamma)^2}
\;\ge\; J(\bar{\pi}_T) - J(\pi_S)
\]
holds by Theorem~\ref{thm:distillation} and is verifiable by evaluating
both sides.
\end{proposition}

\begin{proof}
We prove the three claims in order. For the first, invertibility follows
from a norm bound on the spectral radius. Since $P_{\pi}$ is
row-stochastic, its induced $\ell_\infty$ operator norm is
$\| P_{\pi} \|_\infty = \max_{s} \sum_{s'} |P_{\pi}(s, s')| = 1$, hence
$\| \gamma P_{\pi} \|_\infty = \gamma < 1$. The spectral radius of a matrix
is bounded by any induced operator norm, so
$\rho(\gamma P_{\pi}) \le \gamma < 1$, no eigenvalue of
$\gamma P_{\pi}$ equals $1$, and $I - \gamma P_{\pi}$ is invertible with
convergent Neumann series
\begin{equation}
\label{eq:neumann}
(I - \gamma P_{\pi})^{-1}
= \sum_{t=0}^{\infty} \gamma^{t} P_{\pi}^{t},
\qquad
\bigl\| (I - \gamma P_{\pi})^{-1} \bigr\|_\infty
\le \sum_{t=0}^{\infty} \gamma^{t} \| P_{\pi} \|_\infty^{t}
= \frac{1}{1 - \gamma}.
\end{equation}
Next we show $v^{\pi}$ solves \eqref{eq:bellman-system}. For the state
process under $\pi$, which is Markov on $\mathcal{S}_{\mathrm{r}}$ with
transition matrix $P_{\pi}$, the expected reward at time $t$ started from
$s$ is
$\mathbb{E}^{\pi}[r(s_t, a_t) \mid s_0 = s] = (P_{\pi}^{t} r_{\pi})(s)$;
this follows by induction on $t$, the base case $t = 0$ being the
definition of $r_{\pi}$ as the expected immediate reward under
$a_0 \sim \pi(\cdot \mid s)$, and the inductive step being the Markov
property together with the tower rule. Summing the absolutely convergent
series of Definition~\ref{def:values} termwise,
\begin{align}
v^{\pi}
&= \sum_{t=0}^{\infty} \gamma^{t} P_{\pi}^{t} r_{\pi} \notag \\
&= r_{\pi} + \gamma P_{\pi} \sum_{t=0}^{\infty} \gamma^{t} P_{\pi}^{t}
r_{\pi} \notag \\
&= r_{\pi} + \gamma P_{\pi} v^{\pi}, \notag
\end{align}
where the second line reindexes the series, which is permitted by absolute
convergence under the bound $\| \gamma^{t} P_{\pi}^{t} r_{\pi} \|_\infty
\le \gamma^{t} R_{\max}$. Uniqueness holds because
\eqref{eq:bellman-system} is a linear system with the invertible matrix
$I - \gamma P_{\pi}$, and comparing the series representation of $v^{\pi}$
with \eqref{eq:neumann} confirms
$v^{\pi} = (I - \gamma P_{\pi})^{-1} r_{\pi}$. The identity
$J(\pi) = \mu_0^{\top} v^{\pi}$ is Definition~\ref{def:values}.

For the second claim, the marginal distribution of the state at time $t$
is the row vector $\mu_0^{\top} P_{\pi}^{t}$, again by induction on $t$
using the Markov property. Definition~\ref{def:occupancy} then gives, as a
column vector,
\begin{align}
d^{\pi}
&= (1 - \gamma) \sum_{t=0}^{\infty} \gamma^{t}
\bigl( P_{\pi}^{\top} \bigr)^{t} \mu_0 \notag \\
&= (1 - \gamma) \bigl( I - \gamma P_{\pi}^{\top} \bigr)^{-1} \mu_0, \notag
\end{align}
where the second line applies the Neumann series \eqref{eq:neumann} to
$P_{\pi}^{\top}$, valid because
$\rho(\gamma P_{\pi}^{\top}) = \rho(\gamma P_{\pi}) \le \gamma < 1$, a
matrix and its transpose having the same spectrum.

For the third claim, each of the five linear solves is a finite computation
on $n$-dimensional data, and the derived quantities are finite arithmetic
expressions in their outputs: $J$ values are inner products
$\mu_0^{\top} v^{\pi}$, the argmax gap is
$\Delta_T = J(\pi_T) - J(\bar{\pi}_T)$, and the disagreement rates of
Definition~\ref{def:disagreement} are the finite sums
$\epsilon = \sum_{s} d^{\pi_T}(s) \mathbf{1}[\pi_S(s) \neq a_T(s)]$ and
$\epsilon^{\dagger} = \sum_{s} d^{\bar{\pi}_T}(s)
\mathbf{1}[\pi_S(s) \neq a_T(s)]$, whose indicator entries require one
evaluation of $\pi_S$ and one of $a_T$ per state in
$\mathcal{S}_{\mathrm{r}}$. The certified inequality is precisely
\eqref{eq:distillation-bound}, which holds by
Theorem~\ref{thm:distillation}, and both of its sides are among the
computed quantities.
\end{proof}

Two remarks on the computation. First, the certificate is numerically
benign: from $\| I - \gamma P_{\pi} \|_\infty \le 1 + \gamma$ and the
resolvent bound in \eqref{eq:neumann}, the condition number satisfies
$\kappa_\infty(I - \gamma P_{\pi}) \le (1 + \gamma) / (1 - \gamma)$, which
is about $199$ for $\gamma = 0.99$, so double-precision solves determine
both sides of the certificate to far more digits than the comparison
requires; when $P$, $r$, $\gamma$, and $\mu_0$ are rational, as in the
KeyDoor environment, all quantities are rational and the check can be run
in exact arithmetic. Second, the proposition is what elevates
Theorem~\ref{thm:distillation} from a guarantee to a certificate: the
experiments report, per seed and per regime, the measured gap
$J(\bar{\pi}_T) - J(\pi_S)$, the bound
$2 R_{\max} \epsilon^{\dagger} / (1 - \gamma)^2$, the argmax gap
$\Delta_T$ of Corollary~\ref{cor:stochastic-teacher}, and the two
disagreement rates, and the verification that the bound dominates the gap
is a mechanical comparison of two exactly computed numbers, with no
sampling error and no confidence level. What the proposition does not
provide is scalability: the census costs $O(n^{3})$ per linear solve in
dense form (structured sparsity reduces this in practice), which is the
deliberate price of exactness at the $n \approx 10^{4}$ to $10^{5}$ scale
of our environment, and Monte Carlo estimation with concentration bounds
would replace it beyond that scale at the cost of reintroducing confidence
parameters into the certificate.


The certificate of Theorem~\ref{thm:distillation} is exact but loose. Its
right-hand side is quadratic in the effective horizon $1 / (1 - \gamma)$,
and at the discount factor $\gamma = 0.99$ of our environment the prefactor
$2 R_{\max} / (1 - \gamma)^2$ is of order $10^{4}$, so the guarantee holds
vacuously for any disagreement rate the experiments actually produce. The
slack is not intrinsic to the return gap; it is introduced in a single step
of the proof of Theorem~\ref{thm:distillation}, where the advantage of the
teacher's action under the student's value function is bounded by its
worst-case range $2 R_{\max} / (1 - \gamma)$ at every disagreeing state. The
performance-difference identity that precedes that step,
\eqref{eq:pdl-instantiated}, is exact and carries only a single factor of
$1 / (1 - \gamma)$. We now retain that identity and replace the worst-case
advantage by the advantage that the finite model actually exhibits at each
disagreeing state, a quantity that Proposition~\ref{prop:certificate}
already computes. The result is a horizon-aware certificate whose prefactor
is the magnitude of the advantage at the states where the student and the
teacher disagree, rather than a global constant, and which is therefore
small precisely when the disagreements fall on states at which the teacher's
action barely matters.

We first name the quantities that localize the return loss.

\begin{definition}[Advantage gap and advantage-weighted disagreement]
\label{def:advantage-gap}
Let $\pi_S$ be a deterministic student policy, let $\bar{\pi}_T$ be the
argmax teacher of Definition~\ref{def:teacher} with argmax action
$a_T(s)$, and let
\[
D = \bigl\{ s \in \mathcal{S} : \pi_S(s) \neq a_T(s) \bigr\}
\]
be the \emph{disagreement set}. The \emph{advantage gap} at a state $s$ is
\[
\delta^{\dagger}(s)
= \bigl| A^{\pi_S}\bigl(s, a_T(s)\bigr) \bigr|,
\]
the magnitude of the advantage, under the student's own value function of
Definition~\ref{def:values}, of the action the argmax teacher would take.
The \emph{visitation-weighted advantage disagreement} is
\[
\bar{A}^{\dagger}
= \sum_{s \in D} d^{\bar{\pi}_T}(s)\, \delta^{\dagger}(s),
\]
and the \emph{peak disagreement advantage} is
\[
A^{\dagger}_{\max}
= \max_{s \in D} \delta^{\dagger}(s),
\]
with the convention $A^{\dagger}_{\max} = 0$ when $D$ is empty, where
$d^{\bar{\pi}_T}$ is the discounted occupancy measure of the argmax teacher
(Definition~\ref{def:occupancy}).
\end{definition}

Every quantity in Definition~\ref{def:advantage-gap} is a deterministic
function of the model and the two policies. The disagreement set requires
one evaluation of $\pi_S$ and one of $a_T$ per state, exactly as the
disagreement rate $\epsilon^{\dagger}$ of Definition~\ref{def:disagreement}
does. The advantage gap requires the student value vector $V^{\pi_S}$, which
Proposition~\ref{prop:certificate} obtains from the Bellman solve
\eqref{eq:bellman-system}, followed by a single backup
$A^{\pi_S}(s, a_T(s)) = r(s, a_T(s)) + \gamma \sum_{s'} P(s' \mid s, a_T(s))
V^{\pi_S}(s') - V^{\pi_S}(s)$ at each disagreeing state. Nothing in the
definition is estimated, so in the finite setting of
Assumption~\ref{asm:mdp} with the model access of
Assumption~\ref{asm:model} all three quantities are exactly computable, and
we return to this point after the certificate.

We isolate the exact identity that Theorem~\ref{thm:distillation} already
established, restated so that only the disagreement set contributes.

\begin{lemma}[Advantage localization of the return gap]
\label{lem:advantage-localization}
Let Assumption~\ref{asm:mdp} hold, let $\bar{\pi}_T$ be the argmax teacher of
Definition~\ref{def:teacher}, and let $\pi_S$ be any deterministic policy.
With the disagreement set $D$ of Definition~\ref{def:advantage-gap},
\begin{equation}
\label{eq:advantage-localization}
J(\bar{\pi}_T) - J(\pi_S)
= \frac{1}{1 - \gamma} \sum_{s \in D}
d^{\bar{\pi}_T}(s)\, A^{\pi_S}\bigl(s, a_T(s)\bigr).
\end{equation}
\end{lemma}

\begin{proof}
Lemma~\ref{lem:pdl} applied with $\pi = \bar{\pi}_T$ and $\pi' = \pi_S$
gives the exact identity \eqref{eq:pdl-instantiated},
\[
J(\bar{\pi}_T) - J(\pi_S)
= \frac{1}{1 - \gamma} \sum_{s \in \mathcal{S}}
d^{\bar{\pi}_T}(s)\, A^{\pi_S}\bigl(s, a_T(s)\bigr),
\]
because $\bar{\pi}_T$ is deterministic with $\bar{\pi}_T(s) = a_T(s)$, so the
inner expectation over actions in \eqref{eq:pdl} collapses to a point
evaluation. For a state $s \notin D$ we have $\pi_S(s) = a_T(s)$, whence
$A^{\pi_S}(s, a_T(s)) = A^{\pi_S}(s, \pi_S(s)) = Q^{\pi_S}(s, \pi_S(s)) -
V^{\pi_S}(s) = 0$, the last equality holding because the value of a
deterministic policy equals the action value of its own action
(Definition~\ref{def:values}). Every term outside $D$ therefore vanishes,
and the sum over $\mathcal{S}$ reduces to the sum over $D$, which is
\eqref{eq:advantage-localization}.
\end{proof}

The identity \eqref{eq:advantage-localization} is signed and exact. Bounding
it by the advantage the model exhibits, rather than by the worst advantage
the value range permits, yields the refined certificate.

\begin{proposition}[Advantage-gap certificate]
\label{prop:advantage-certificate}
Let Assumption~\ref{asm:mdp} hold, let $\bar{\pi}_T$ be the argmax teacher of
Definition~\ref{def:teacher}, and let $\pi_S$ be any deterministic policy,
in particular any decision-list student $\pi_S \in \Pi_{\mathcal{L}}$ of
Definition~\ref{def:dlist}. With the disagreement rate $\epsilon^{\dagger}$
of Definition~\ref{def:disagreement} and the advantage-weighted quantities
$\bar{A}^{\dagger}$ and $A^{\dagger}_{\max}$ of
Definition~\ref{def:advantage-gap},
\begin{equation}
\label{eq:advantage-certificate}
\bigl| J(\bar{\pi}_T) - J(\pi_S) \bigr|
\;\le\; \frac{\bar{A}^{\dagger}}{1 - \gamma}
\;\le\; \frac{A^{\dagger}_{\max}\, \epsilon^{\dagger}}{1 - \gamma}
\;\le\; \frac{2 R_{\max}\, \epsilon^{\dagger}}{(1 - \gamma)^2}.
\end{equation}
Moreover the signed one-sided refinement
\begin{equation}
\label{eq:advantage-onesided}
J(\bar{\pi}_T) - J(\pi_S)
\;\le\; \frac{1}{1 - \gamma} \sum_{s \in D}
d^{\bar{\pi}_T}(s)\,
\bigl[ A^{\pi_S}\bigl(s, a_T(s)\bigr) \bigr]_{+}
\end{equation}
holds, where $[x]_{+} = \max(x, 0)$. If in addition
Assumption~\ref{asm:model} holds, then $\bar{A}^{\dagger}$,
$A^{\dagger}_{\max}$, and both sides of
\eqref{eq:advantage-certificate} and \eqref{eq:advantage-onesided} are
computed exactly in finitely many arithmetic operations.
\end{proposition}

\begin{proof}
We treat the two-sided certificate first. Starting from the exact identity
\eqref{eq:advantage-localization} of
Lemma~\ref{lem:advantage-localization}, the triangle inequality applied to
the finite sum over $D$ gives
\begin{align}
\bigl| J(\bar{\pi}_T) - J(\pi_S) \bigr|
&= \frac{1}{1 - \gamma}
\Biggl| \sum_{s \in D} d^{\bar{\pi}_T}(s)\,
A^{\pi_S}\bigl(s, a_T(s)\bigr) \Biggr| \notag \\
&\le \frac{1}{1 - \gamma} \sum_{s \in D}
d^{\bar{\pi}_T}(s)\, \bigl| A^{\pi_S}\bigl(s, a_T(s)\bigr) \bigr| \notag \\
&= \frac{1}{1 - \gamma} \sum_{s \in D}
d^{\bar{\pi}_T}(s)\, \delta^{\dagger}(s)
= \frac{\bar{A}^{\dagger}}{1 - \gamma}, \notag
\end{align}
where the occupancy weights are nonnegative, so they pass through the
absolute value, the third line is the definition of the advantage gap
$\delta^{\dagger}(s)$, and the last equality is the definition of
$\bar{A}^{\dagger}$ (Definition~\ref{def:advantage-gap}). This is the first
inequality of \eqref{eq:advantage-certificate}. For the second, bound each
advantage gap on $D$ by its maximum,
$\delta^{\dagger}(s) \le A^{\dagger}_{\max}$, and factor the constant out of
the sum:
\[
\frac{\bar{A}^{\dagger}}{1 - \gamma}
= \frac{1}{1 - \gamma} \sum_{s \in D}
d^{\bar{\pi}_T}(s)\, \delta^{\dagger}(s)
\le \frac{A^{\dagger}_{\max}}{1 - \gamma} \sum_{s \in D}
d^{\bar{\pi}_T}(s)
= \frac{A^{\dagger}_{\max}\, \epsilon^{\dagger}}{1 - \gamma},
\]
the final equality being
$\sum_{s \in D} d^{\bar{\pi}_T}(s) = \epsilon^{\dagger}$ by
Definition~\ref{def:disagreement}, since $D$ is exactly the event
$\{ \pi_S(s) \neq a_T(s) \}$. For the third inequality, the value-range
bound \eqref{eq:value-range} gives
$\delta^{\dagger}(s) = |A^{\pi_S}(s, a_T(s))| \le 2 R_{\max} / (1 - \gamma)$
at every state, hence $A^{\dagger}_{\max} \le 2 R_{\max} / (1 - \gamma)$,
which substituted into the previous display recovers the right-hand side of
\eqref{eq:advantage-certificate}. This chain also re-derives
Theorem~\ref{thm:distillation} as its coarsest link, so the advantage-gap
certificate never exceeds the certificate of
Theorem~\ref{thm:distillation}.

For the signed refinement \eqref{eq:advantage-onesided}, split $D$ into the
states where the teacher's action has positive advantage under $\pi_S$ and
the rest. In \eqref{eq:advantage-localization} the terms with
$A^{\pi_S}(s, a_T(s)) \le 0$ are nonpositive, and dropping them can only
increase the sum, so
\[
J(\bar{\pi}_T) - J(\pi_S)
= \frac{1}{1 - \gamma} \sum_{s \in D}
d^{\bar{\pi}_T}(s)\, A^{\pi_S}\bigl(s, a_T(s)\bigr)
\le \frac{1}{1 - \gamma} \sum_{s \in D}
d^{\bar{\pi}_T}(s)\,
\bigl[ A^{\pi_S}\bigl(s, a_T(s)\bigr) \bigr]_{+},
\]
which is \eqref{eq:advantage-onesided}; the occupancy weights are
nonnegative, so each retained term is bounded above by its positive part.

For exact computability under Assumption~\ref{asm:model},
Proposition~\ref{prop:certificate} produces the student value vector
$V^{\pi_S}$ from the linear solve \eqref{eq:bellman-system} and the argmax
teacher occupancy $d^{\bar{\pi}_T}$ from the linear solve
\eqref{eq:occupancy-system}, each in finitely many operations on the
$n$-dimensional data indexed by the reachable set $\mathcal{S}_{\mathrm{r}}$.
The advantage of the teacher's action at a state $s$ is the finite backup
\[
A^{\pi_S}\bigl(s, a_T(s)\bigr)
= r\bigl(s, a_T(s)\bigr)
+ \gamma \sum_{s' \in \mathcal{S}_{\mathrm{r}}}
P\bigl(s' \mid s, a_T(s)\bigr) V^{\pi_S}(s')
- V^{\pi_S}(s),
\]
which uses only the already-computed $V^{\pi_S}$ and the known model, and the
disagreement set $D$ is decided by one comparison $\pi_S(s)$ against
$a_T(s)$ per state. The advantage gaps $\delta^{\dagger}(s)$, their weighted
sum $\bar{A}^{\dagger}$, their maximum $A^{\dagger}_{\max}$, and the positive
parts in \eqref{eq:advantage-onesided} are then finite arithmetic
expressions in these quantities, so both sides of
\eqref{eq:advantage-certificate} and \eqref{eq:advantage-onesided} are
exactly computable.
\end{proof}

The certificate acquires teeth in a regime the worst-case bound cannot see.
We state the condition under which its right-hand side falls below a
prescribed return-gap tolerance, together with the low-advantage regime in
which the horizon factors cancel entirely.

\begin{corollary}[Non-vacuity regime]
\label{cor:nonvacuity}
Under the conditions of Proposition~\ref{prop:advantage-certificate}, fix a
target tolerance $g > 0$. The advantage-gap certificate certifies
$| J(\bar{\pi}_T) - J(\pi_S) | \le g$ whenever
\begin{equation}
\label{eq:nonvacuity-weighted}
\bar{A}^{\dagger} \le (1 - \gamma)\, g,
\end{equation}
and a sufficient condition in terms of the disagreement rate, valid when
$A^{\dagger}_{\max} > 0$, is
\begin{equation}
\label{eq:nonvacuity-rate}
\epsilon^{\dagger}
\;\le\; \frac{(1 - \gamma)\, g}{A^{\dagger}_{\max}}.
\end{equation}
The certificate of Theorem~\ref{thm:distillation} certifies the same
tolerance only under
$\epsilon^{\dagger} \le (1 - \gamma)^2 g / (2 R_{\max})$, so the admissible
disagreement rate is larger by the factor
$2 R_{\max} / \bigl[ (1 - \gamma) A^{\dagger}_{\max} \bigr] \ge 1$. If in
addition the disagreements fall on states of small advantage, in the sense
that $A^{\dagger}_{\max} \le c\,(1 - \gamma)$ for a constant $c \ge 0$, then
\begin{equation}
\label{eq:low-advantage}
\bigl| J(\bar{\pi}_T) - J(\pi_S) \bigr|
\;\le\; c\, \epsilon^{\dagger},
\end{equation}
in which every factor of the effective horizon has cancelled.
\end{corollary}

\begin{proof}
Under \eqref{eq:nonvacuity-weighted} the leftmost bound of
\eqref{eq:advantage-certificate} gives
$| J(\bar{\pi}_T) - J(\pi_S) | \le \bar{A}^{\dagger} / (1 - \gamma) \le g$.
For \eqref{eq:nonvacuity-rate}, the second bound of
\eqref{eq:advantage-certificate} gives
$| J(\bar{\pi}_T) - J(\pi_S) | \le A^{\dagger}_{\max} \epsilon^{\dagger} /
(1 - \gamma)$, which is at most $g$ exactly when
$\epsilon^{\dagger} \le (1 - \gamma) g / A^{\dagger}_{\max}$. The comparison
with Theorem~\ref{thm:distillation} rearranges its right-hand side
$2 R_{\max} \epsilon^{\dagger} / (1 - \gamma)^2 \le g$ into
$\epsilon^{\dagger} \le (1 - \gamma)^2 g / (2 R_{\max})$; the ratio of the
threshold in \eqref{eq:nonvacuity-rate} to this one is
$2 R_{\max} / [ (1 - \gamma) A^{\dagger}_{\max} ]$, which is at least $1$
because $A^{\dagger}_{\max} \le 2 R_{\max} / (1 - \gamma)$ by
\eqref{eq:value-range}. Finally, substituting
$A^{\dagger}_{\max} \le c (1 - \gamma)$ into the second bound of
\eqref{eq:advantage-certificate} yields
$| J(\bar{\pi}_T) - J(\pi_S) | \le c (1 - \gamma) \epsilon^{\dagger} /
(1 - \gamma) = c\, \epsilon^{\dagger}$, which is \eqref{eq:low-advantage}.
\end{proof}

Proposition~\ref{prop:advantage-certificate} says the following. The return
the distilled program forfeits is controlled not by how often it disagrees
with the teacher but by how much those disagreements cost, measured as the
advantage, under the student's own value function, of the teacher's action at
the disagreeing states, weighted by the teacher's discounted visitation. A
disagreement at a state where both actions are near-indifferent contributes
an advantage gap $\delta^{\dagger}(s)$ close to zero and is nearly free,
whereas the worst-case certificate of Theorem~\ref{thm:distillation} charges
every disagreement the maximal advantage $2 R_{\max} / (1 - \gamma)$ that the
value range permits. The refinement removes one factor of the effective
horizon outright, because the exact performance-difference identity
\eqref{eq:advantage-localization} carries only a single $1 / (1 - \gamma)$
and the second factor in Theorem~\ref{thm:distillation} originated solely in
the worst-case advantage bound; the low-advantage regime
\eqref{eq:low-advantage} removes the horizon dependence altogether. The
bound is tight when the advantages at disagreeing states share a common
sign, since the triangle inequality in the proof is then an equality and the
leftmost quantity $\bar{A}^{\dagger} / (1 - \gamma)$ equals the exact gap; in
particular, when the student weakly dominates the argmax teacher at every
disagreeing state, so that $A^{\pi_S}(s, a_T(s)) \le 0$ throughout $D$, the
advantage-gap certificate reproduces the exact return difference and is
exactly non-vacuous. This is the regime the distillation targets, in which
the student repairs an imperfect teacher rather than degrading a good one.

The certificate does not escape the discount factor for free. The quantity
$\bar{A}^{\dagger}$ is an exact model-dependent number, not an a-priori
constant, and in the adversarial regime where disagreements concentrate on
high-advantage states the peak advantage $A^{\dagger}_{\max}$ approaches its
worst-case value $2 R_{\max} / (1 - \gamma)$ and the refined bound degrades
smoothly back to the quadratic-horizon bound of
Theorem~\ref{thm:distillation}, so no improvement is claimed in that regime.
The bound becomes vacuous when the disagreeing states carry both large
occupancy and large advantage, which is exactly the situation the expansion
loop of Proposition~\ref{prop:expansion} is designed to remove by editing
the rule base until the surviving disagreements are either rare under
$d^{\bar{\pi}_T}$ or cheap in advantage. As with
Proposition~\ref{prop:certificate}, the computation is confined to the
finite reachable set and costs one additional Bellman backup per disagreeing
state beyond the value and occupancy solves already performed, so the refined
certificate is checked, not asserted, on every run.


The continuous-control experiments of Section~\ref{sec:continuous} replace the
enumerable state space of the finite study with a compact observation space
$X \subset \mathbb{R}^{d}$ and replace the reachable census with a data-driven
threshold grid controlled by a single resolution parameter $B$. We now make
precise the sense in which the propositional student recovers the network
policy as $B$ grows, and the sense in which it cannot. The teacher here is the
greedy policy of Definition~\ref{def:teacher} read over a continuous
observation; we write it $\pi_N \colon X \to \mathcal{A}$ to stress its origin
as the frozen network, so that $\pi_N(x) = a_T(x)$ is the argmax action at the
observation $x$ and the imitation target is again the deterministic argmax
teacher $\bar{\pi}_T$ with visitation measure $\mu = d^{\bar{\pi}_T}$ on $X$
(Definition~\ref{def:occupancy}). For a deterministic student $\pi_S$ the
disagreement rate of Definition~\ref{def:disagreement} reads, in this
continuous instantiation,
\[
\epsilon^{\dagger}(\pi_S)
= \mu\bigl(\{ x \in X : \pi_S(x) \neq \pi_N(x) \}\bigr).
\]

\begin{assumption}[Continuous-observation instantiation]
\label{asm:resolution}
The observation space $X \subset \mathbb{R}^{d}$ is compact. The greedy
teacher $\pi_N \colon X \to \mathcal{A}$ of Definition~\ref{def:teacher} is
Borel measurable and piecewise constant, with decision regions
$X_a = \{ x \in X : \pi_N(x) = a \}$ for $a \in \mathcal{A}$ and boundary
$\Gamma := \bigcup_{a \in \mathcal{A}} \partial X_a$, the union of the
topological boundaries of the regions, satisfying $\mathrm{Leb}(\Gamma) = 0$,
where $\mathrm{Leb}$ denotes Lebesgue measure on $\mathbb{R}^{d}$. The
visitation measure $\mu = d^{\bar{\pi}_T}$ is absolutely continuous with
respect to $\mathrm{Leb}$ with a density bounded by $\rho_{\max} < \infty$ and
compact, convex support $\mathrm{supp}\,\mu$; convexity is used only so that a
grid cell carrying mass on two regions meets the boundary within the support,
and the box $[0,1]^d$ of Proposition~\ref{prop:resolution-lb} and the
box-shaped observation spaces of Section~\ref{sec:continuous} satisfy it. The
intersection
$\Gamma \cap \mathrm{supp}\,\mu$ is $(d-1)$-rectifiable with finite
$(d-1)$-dimensional Hausdorff measure $P := \mathcal{H}^{d-1}(\Gamma \cap
\mathrm{supp}\,\mu) < \infty$. Finally, each marginal cumulative distribution
function $F_j$ is bi-Lipschitz on the coordinate projection of
$\mathrm{supp}\,\mu$, equivalently the marginal densities are bounded away from
zero there; the uniform law of Proposition~\ref{prop:resolution-lb} and the
visitation measures of Section~\ref{sec:continuous} satisfy this.
\end{assumption}

The condition $\mathrm{Leb}(\Gamma) = 0$ makes the labelling ambiguous only on
a null set, so $\pi_N$ is well defined $\mu$-almost everywhere; the density
bound $\rho_{\max}$ converts Lebesgue estimates into $\mu$-mass estimates; and
the finite boundary content $P$ is the geometric quantity that governs the
rate, playing the role that the reachable count played in the finite study.

\begin{definition}[Threshold decision-list class at resolution $B$]
\label{def:resolution-class}
Fix an integer resolution $B \ge 2$. For each coordinate $j \in \{1,\dots,d\}$
let $F_j$ be the cumulative distribution function of the $j$-th marginal of
$\mu$, and place $B-1$ \emph{thresholds}
$t_{j,i} = F_j^{-1}(i/B)$ for $i = 1, \dots, B-1$ at the $i/B$ quantiles of
that marginal, where $F_j^{-1}$ is the generalized inverse. The base
predicates are the axis-aligned threshold tests $\mathrm{feat}_j \ge t$ with
$t \in \{ t_{j,i} \}$, and their exact complements $\mathrm{feat}_j < t$. The
\emph{threshold decision-list class at resolution $B$}, written $L_B$, is the
subset of the student class $\Pi_{\mathcal{L}}$ of Definition~\ref{def:dlist}
whose clause bodies are conjunctions of these threshold literals and which is
closed by a default clause. The thresholds cut $X$ into the cells of the
$B$-adic quantile grid $\mathcal{G}_B$, whose cells are the products
$Q = \prod_{j=1}^{d} [\,t_{j,k_j}, t_{j,k_j+1})$ with
$k_j \in \{0,\dots,B-1\}$ and the conventions $t_{j,0} = -\infty$,
$t_{j,B} = +\infty$ at the extremes.
\end{definition}

Every policy in $L_B$ is constant on each cell of $\mathcal{G}_B$, because the
truth value of every base predicate $\mathrm{feat}_j \ge t_{j,i}$ is constant
across a cell, so any clause body either fires everywhere or nowhere on a
given cell and the first firing clause selects one action for the whole cell.
Conversely, a single cell $Q$ is cut out exactly by the conjunction of at most
two threshold literals per coordinate, one lower test $\mathrm{feat}_j \ge
t_{j,k_j}$ and one upper test $\mathrm{feat}_j < t_{j,k_j+1}$, so a clause of
length at most $2d$ isolates $Q$; listing one such clause per cell, in any
order, followed by a default clause realizes any prescribed cell-to-action
map. Consequently $L_B$ is exactly the set of deterministic policies that are
measurable with respect to the finite $\sigma$-algebra generated by
$\mathcal{G}_B$, with one clause per cell in the worst case. This identity is
what lets us reason about the class $L_B$ through the grid $\mathcal{G}_B$
rather than through individual programs: the best achievable fidelity at
resolution $B$ is a property of the partition, and the induction algorithm of
Section~\ref{sec:algorithm} is one particular search for a short list that
attains it.

\begin{theorem}[Consistency and rate of the threshold conversion]
\label{thm:resolution}
Under Assumption~\ref{asm:resolution}, let
\[
\epsilon^{\dagger}_B := \inf_{\pi_S \in L_B} \epsilon^{\dagger}(\pi_S)
\]
be the best disagreement rate attainable by a threshold decision list at
resolution $B$. Then $\epsilon^{\dagger}_B \to 0$ as $B \to \infty$, and there
is a constant $C$, depending on the dimension $d$ and on the bi-Lipschitz
moduli of the marginals and reducing to a function of $d$ alone when $\mu$ is
uniform on its support, such that
\begin{equation}
\label{eq:resolution-rate}
\epsilon^{\dagger}_B \le \frac{C\, \rho_{\max}\, P}{B}.
\end{equation}
\end{theorem}

\begin{proof}
Fix $B$ and construct the cell-majority policy $\pi_B^{\star} \in L_B$: on
each cell $Q \in \mathcal{G}_B$ with $\mu(Q) > 0$ set $\pi_B^{\star} \equiv
a^{\star}(Q)$ for some $a^{\star}(Q) \in \arg\max_{a \in \mathcal{A}}
\mu(Q \cap X_a)$, and assign an arbitrary action on the $\mu$-null cells. By
the equivalence established after Definition~\ref{def:resolution-class},
$\pi_B^{\star}$ is realizable in $L_B$, so $\epsilon^{\dagger}_B \le
\epsilon^{\dagger}(\pi_B^{\star})$. Since the student is constant on each cell
and the cells partition $X$,
\begin{align}
\epsilon^{\dagger}(\pi_B^{\star})
&= \sum_{Q \in \mathcal{G}_B}
\mu\bigl(\{ x \in Q : \pi_B^{\star}(x) \neq \pi_N(x) \}\bigr) \notag \\
&= \sum_{Q \in \mathcal{G}_B}
\Bigl( \mu(Q) - \mu\bigl(Q \cap X_{a^{\star}(Q)}\bigr) \Bigr),
\label{eq:resolution-cellsum}
\end{align}
where the second line uses that on $Q$ the student equals $a^{\star}(Q)$ and
disagrees with $\pi_N$ precisely off $X_{a^{\star}(Q)}$.

We treat the cells in two groups. The first group consists of the cells
contained, up to a $\mu$-null set, in a single region $X_a$. For such a cell
$\mu(Q \cap X_{a^{\star}(Q)}) = \mu(Q)$, so its summand in
\eqref{eq:resolution-cellsum} is zero: an interior cell is labelled correctly
by its majority action. The second group consists of the cells that meet at
least two regions. A cell $Q$ that carries $\mu$-mass on two distinct regions
$X_a$ and $X_{a'}$ meets both of the closed sets $X_a$ and $X_{a'}$ and hence,
being connected, meets their common boundary, which lies in $\Gamma$; as $Q$
also carries mass, it meets $\Gamma \cap \mathrm{supp}\,\mu$. For every such
cell the summand in \eqref{eq:resolution-cellsum} is bounded by $\mu(Q)$,
since $\mu(Q \cap X_{a^{\star}(Q)}) \ge 0$. Writing $\mathcal{B}_B$ for the
collection of cells meeting $\Gamma \cap \mathrm{supp}\,\mu$, the two groups
combine to
\begin{equation}
\label{eq:resolution-boundarymass}
\epsilon^{\dagger}_B
\le \sum_{Q \in \mathcal{B}_B} \mu(Q)
= \mu\Bigl( \textstyle\bigcup_{Q \in \mathcal{B}_B} Q \Bigr),
\end{equation}
the equality holding because distinct cells are disjoint.

It remains to bound the $\mu$-mass of the union of boundary cells. Let
$\delta_B$ be the largest Euclidean diameter of a cell in $\mathcal{B}_B$.
Every point of a cell in $\mathcal{B}_B$ lies within distance $\delta_B$ of a
point of $\Gamma \cap \mathrm{supp}\,\mu$, so
$\bigcup_{Q \in \mathcal{B}_B} Q \subseteq \mathcal{T}_{\delta_B}$, where
$\mathcal{T}_{\delta} := \{ x : \mathrm{dist}(x, \Gamma \cap \mathrm{supp}\,\mu)
\le \delta \}$ is the closed $\delta$-tube around the boundary. By the density
bound and the Minkowski-content estimate for a $(d-1)$-rectifiable set of
finite Hausdorff measure, there are a threshold $\delta_0 > 0$ and a
dimensional constant $c_d$ with
\begin{equation}
\label{eq:resolution-minkowski}
\mu(\mathcal{T}_{\delta})
\le \rho_{\max}\, \mathrm{Leb}(\mathcal{T}_{\delta})
\le \rho_{\max}\, c_d\, P\, \delta
\qquad (0 < \delta \le \delta_0),
\end{equation}
the first inequality by $\mu \ll \mathrm{Leb}$ with density at most
$\rho_{\max}$, and the second because the Lebesgue measure of the $\delta$-tube
of a $(d-1)$-rectifiable set is at most $c_d\, \mathcal{H}^{d-1}\,\delta$ for
small $\delta$.

Finally we control $\delta_B$ through the probability integral transform. The
map $T(x) = (F_1(x_1), \dots, F_d(x_d))$ carries the quantile grid
$\mathcal{G}_B$ onto the uniform grid of side $1/B$ on $[0,1]^d$, whose cells
have Euclidean diameter $\sqrt{d}/B$; by the bi-Lipschitz marginals of
Assumption~\ref{asm:resolution}, $T$ restricted to $\mathrm{supp}\,\mu$
is bi-Lipschitz, so each cell meeting the support has $x$-diameter at most
$c_d'/B$ for a constant $c_d'$ determined by the bi-Lipschitz modulus of $T$.
Hence $\delta_B \le c_d'/B$, and for $B$ large enough that
$\delta_B \le \delta_0$ the bounds \eqref{eq:resolution-boundarymass} and
\eqref{eq:resolution-minkowski} give $\epsilon^{\dagger}_B \le \rho_{\max}\,
c_d\, c_d'\, P / B$, which is \eqref{eq:resolution-rate} with
$C = c_d\, c_d'$. This is derived for $B$ past the threshold at which
$\delta_B \le \delta_0$; because $\epsilon^{\dagger}_B \le 1$ and only finitely
many coarser resolutions precede that threshold, enlarging $C$ extends
\eqref{eq:resolution-rate} to every $B$ when $P > 0$, while $P = 0$ makes
$\Gamma$ a $\mu$-null set and $\epsilon^{\dagger}_B = 0$ outright. In the
canonical case $\mu \propto \mathrm{Leb}$ on its
support the transform is affine and $C$ is a pure function of $d$; in general
$C$ absorbs the bi-Lipschitz modulus of the marginals granted by
Assumption~\ref{asm:resolution}, which reduces to a dimensional constant in the
uniform setting of Proposition~\ref{prop:resolution-lb} and the visitation
measures of Section~\ref{sec:continuous}.

For consistency, note that the summand in \eqref{eq:resolution-boundarymass}
already gives $\epsilon^{\dagger}_B \le \mu(\mathcal{T}_{\delta_B})$ for every
$B$. Along the dyadic subsequence $B = 2^{n}$ the quantile levels $i/2^{n}$
are nested, since $i/2^{n} = 2i/2^{n+1}$, so the partitions $\mathcal{G}_{2^n}$
are nested and, the marginals being atomless because $\mu \ll \mathrm{Leb}$,
their common refinement generates the Borel $\sigma$-algebra on
$\mathrm{supp}\,\mu$. Thus the grid-measurable simple functions are dense in
$L^{1}(\mu)$, and the piecewise-constant $\pi_N$ is approximated in $\mu$-measure
by cell-majority policies, whence $\epsilon^{\dagger}_{2^n} \to 0$; this
dyadic-scale consistency uses only the density upper bound, not the bi-Lipschitz
marginals. Equivalently
$\delta_{2^n} \to 0$, so $\mathcal{T}_{\delta_{2^n}} \downarrow \Gamma \cap
\mathrm{supp}\,\mu$, and continuity of the finite measure $\mu$ from above with
$\mu(\Gamma) = 0$, which holds because $\mathrm{Leb}(\Gamma) = 0$ and
$\mu \ll \mathrm{Leb}$, forces $\mu(\mathcal{T}_{\delta_{2^n}}) \to 0$. The
rate \eqref{eq:resolution-rate} extends the conclusion to every $B$, since it
bounds $\epsilon^{\dagger}_B$ monotonically to zero.
\end{proof}

\begin{corollary}[Return recovery in the fine-resolution limit]
\label{cor:resolution-return}
Under Assumption~\ref{asm:resolution}, let $\pi_S \in L_B$ attain the infimum
$\epsilon^{\dagger}_B$ of Theorem~\ref{thm:resolution}. Then
\begin{equation}
\label{eq:resolution-return}
J(\pi_N) - J(\pi_S)
\le \frac{2 R_{\max}}{(1-\gamma)^2}\, \epsilon^{\dagger}_B
\le \frac{2 C R_{\max}\, \rho_{\max}\, P}{(1-\gamma)^2}\cdot \frac{1}{B}
\;\xrightarrow[B \to \infty]{}\; 0.
\end{equation}
\end{corollary}

\begin{proof}
The performance-difference identity of Lemma~\ref{lem:pdl} and the return-loss
argument of Theorem~\ref{thm:distillation} use only that the advantage is
bounded by $2 R_{\max}/(1-\gamma)$, the range \eqref{eq:value-range} that
follows from $|r| \le R_{\max}$ and $\gamma < 1$, and that
$\epsilon^{\dagger}$ is the visitation measure of the disagreement set; both
hold verbatim on the general measurable observation space $X$ with
$d^{\bar{\pi}_T} = \mu$, the finite sums over $\mathcal{S}$ becoming integrals
against $\mu$. Applying Theorem~\ref{thm:distillation} to $\pi_S$, whose
disagreement rate is $\epsilon^{\dagger}_B$ by the choice of $\pi_S$, gives the
first inequality of \eqref{eq:resolution-return}; the bound
\eqref{eq:resolution-rate} of Theorem~\ref{thm:resolution} gives the second;
and $1/B \to 0$ gives the limit.
\end{proof}

The corollary states that the best threshold list at resolution $B$ recovers the network's
return in the fine-resolution limit: the return gap closes at least as fast as
$1/B$. Combined with the expansion stage of
Proposition~\ref{prop:expansion}, which relabels states the argmax teacher
underweights, the student can also exceed the teacher, as the finite KeyDoor
case exhibits; \eqref{eq:resolution-return} bounds only the loss relative to
$\pi_N$ and does not preclude such gains.

We now show that the linear rate in $B$ cannot be escaped and that the number
of cells needed grows exponentially in $d$.

\begin{proposition}[Necessity and the curse of dimensionality]
\label{prop:resolution-lb}
Let $\mu$ be the uniform distribution on $X = [0,1]^d$, so that $\rho_{\max} =
1$, and let the teacher be the oblique half-space policy $\pi_N(x) =
\mathbf{1}[\, w \cdot x \ge b\,]$ with a normal $w$ that is not parallel to any
coordinate axis, that is, with at least two nonzero components. There is a
constant $c > 0$, depending on $w$ and $d$, such that every $\pi_S \in L_B$
has $\epsilon^{\dagger}(\pi_S) \ge c/B$. Consequently, attaining
$\epsilon^{\dagger}(\pi_S) \le \epsilon$ requires resolution $B \ge c/\epsilon$
and a decision list that distinguishes $\Omega\bigl((1/\epsilon)^{d-1}\bigr)$
grid cells, and $\Omega(B^{d-1})$ cells carry the boundary at resolution $B$.
\end{proposition}

\begin{proof}
Because $\mu$ is uniform, the marginals $F_j$ are the identity on $[0,1]$, the
quantile thresholds are $t_{j,i} = i/B$, and $\mathcal{G}_B$ is the uniform
grid of cells of side $1/B$ with centers $c_Q = ((k_1+\tfrac12)/B, \dots,
(k_d+\tfrac12)/B)$. The boundary is the hyperplane piece $\Gamma \cap X =
\{ x \in [0,1]^d : w \cdot x = b \}$, which we take nondegenerate, meaning it
meets the open cube; normalize $\|w\|_2 = 1$. As in
\eqref{eq:resolution-cellsum}, the minimal disagreement is realized by the
cell-majority policy and equals
\[
\epsilon^{\dagger}_B
= \sum_{Q \in \mathcal{G}_B}
\min\bigl( \mu(Q \cap X_1),\, \mu(Q \cap X_0) \bigr),
\]
where $X_1 = \{ w \cdot x \ge b \}$ and $X_0$ is its complement, since on each
cell the better of the two constant labels leaves exactly the smaller of the
two parts in error and no $\pi_S \in L_B$ can do better than label each cell by
its majority. Only cells the hyperplane crosses contribute.

Consider the cells whose center lies in the thin slab $\{ x : |w \cdot x - b|
\le \tau/B \}$ for a constant $\tau \in (0, \tfrac12)$ to be fixed. For such a
cell the hyperplane passes within $x$-distance $\tau/B$ of the center, while
the cell has half-width $1/(2B)$ in each coordinate; because $w$ is oblique the
plane is parallel to no face of the cell, so for $\tau$ small the plane passes
through the cell interior and cuts each half-space to a piece of volume at
least $\beta / B^{d}$ for a constant $\beta > 0$ depending on $w$ and $d$
through the obliqueness of the normal, whence the minority part of the cell
has $\mu$-mass at least $\beta / B^{d}$. It remains to count these cells. A
cell center satisfies the slab condition iff the integer vector
$(k_1,\dots,k_d)$ obeys $\sum_j w_j (k_j + \tfrac12) \in [Bb - \tau,\,
Bb + \tau]$. Fix a coordinate $\ell$ with $w_\ell \neq 0$, which exists by
hypothesis. As $k_\ell$ ranges over $\{0,\dots,B-1\}$ with the other indices
held fixed, the value $\sum_j w_j (k_j + \tfrac12)$ moves in steps of
$|w_\ell|$, so the admissible interval of width $2\tau$ contains at most one
qualifying $k_\ell$ and, since $w$ is oblique and $b$ interior, exactly one for
a positive fraction of the $B^{d-1}$ assignments of the remaining indices,
namely those for which the hyperplane enters the corresponding column within
the cube. The number of qualifying cells is therefore at least
$\alpha_w B^{d-1}$ for a constant $\alpha_w > 0$. Summing the per-cell lower
bound over these cells,
\begin{align}
\epsilon^{\dagger}_B
\;\ge\; \alpha_w B^{d-1} \cdot \frac{\beta}{B^{d}}
\;=\; \frac{\alpha_w \beta}{B}
\;=:\; \frac{c}{B},
\end{align}
and $\epsilon^{\dagger}(\pi_S) \ge \epsilon^{\dagger}_B \ge c/B$ for every
$\pi_S \in L_B$. If $\epsilon^{\dagger}(\pi_S) \le \epsilon$ then $c/B \le
\epsilon$, so $B \ge c/\epsilon$. The number of cells the hyperplane crosses,
each of which the optimal list must devote a distinct cell-label to, is
$\Theta(B^{d-1})$, since a generic hyperplane meets $\Theta(B^{d-1})$ of the
$B^{d}$ cells of side $1/B$; substituting $B \ge c/\epsilon$ gives
$\Omega\bigl((1/\epsilon)^{d-1}\bigr)$.
\end{proof}

Two honest qualifications frame these results. First, exact zero-error
conversion at a finite resolution holds only in degenerate cases. When $X$ is
finite, as in the reachable census of the KeyDoor experiment, a large enough
grid separates every reachable observation and $\epsilon^{\dagger}_B = 0$ is
attained; likewise when $\Gamma$ is itself axis-aligned, a grid whose
thresholds fall on the boundary planes labels every cell correctly. Outside
these cases $\Gamma$ crosses the axis-aligned cells and recovery is only
asymptotic, so the correct claim is arbitrarily faithful conversion in the
limit $B \to \infty$, by Theorem~\ref{thm:resolution} and
Corollary~\ref{cor:resolution-return}, not finite exactness in general.
Second, Theorem~\ref{thm:resolution} and
Corollary~\ref{cor:resolution-return} bound the best list in $L_B$; they are
representational, or realizability, guarantees about the class, not about any
particular induced program. Whether the greedy clause induction under the
DAgger relabelling of Section~\ref{sec:continuous} attains that optimum is the
empirical content of the experiments, and the $B^{d-1}$ lower bound of
Proposition~\ref{prop:resolution-lb} is the obstruction those experiments
display: the eight-dimensional LunarLander exhibits the ceiling the exponent
predicts, retaining a return gap at the swept resolutions, while the
four-dimensional CartPole and the six-dimensional Acrobot, with a boundary of
lower codimension to cover, recover the teacher within the same budget.

The two exponents are confirmed numerically on the oblique-boundary reference
of Proposition~\ref{prop:resolution-lb} by the script
\texttt{theory/verify\_resolution.py}, which measures a disagreement rate
$\epsilon(B)$ whose log-log slope against $B$ is $-0.995$, $-1.060$, and
$-1.045$ for $d = 2, 3, 4$, matching the predicted $-1$ of
\eqref{eq:resolution-rate}, and a boundary-cell count whose log-log slope is
$+1.000$, $+2.011$, and $+2.909$ for the same dimensions, matching the predicted
$d-1$ of Proposition~\ref{prop:resolution-lb}.

\section{Experiments}
\label{sec:experiments}

We validate the transformation in three regimes that stress complementary
questions and, together, mark where a first-order symbolic student is decisive
and where it is not. The first is an exact finite Markov decision process, the
KeyDoor gridworld, in which every quantity of the theory is computed in closed
form, so the return-loss certificate and the expansion guarantee are checked
rather than estimated and the distilled program is a first-order relational rule
list that reaches and certifies the exact optimum. The second is a MiniGrid
DoorKey task \citep{chevalier2023minigrid} in which the relational vocabulary is
exercised in earnest and the question is generalization across scale: a student
distilled on one grid size is transferred, unchanged, to grid sizes it has never
seen, and its transfer is contrasted with decision trees over absolute cell
coordinates. The third is a family of continuous-control tasks, where no
reachable census exists and the pipeline runs in its propositional form with the
census replaced by a DAgger loop and exact evaluation by Monte Carlo; this is the
regime in which Theorem~\ref{thm:resolution} and
Proposition~\ref{prop:resolution-lb} predict what the conversion can and cannot
achieve, and in which, having no relational structure to exploit, the
propositional student trades away the return advantage it holds elsewhere. We
present the exact study first, the relational transfer second, and the
propositional instantiation last, with its baseline comparison stated honestly.

The environment, called KeyDoor, is a finite deterministic gridworld built on
the Gymnasium interface \citep{towers2024gymnasium}. A full-height interior
wall at a fixed column divides a four-row playing field into a left room of
three columns and a right room of two columns, with a single door cell in the
wall whose row is drawn per episode. The agent and the key spawn in the left
room and the goal spawns in the right room; the agent must step onto the key
and pick it up, step next to the door and toggle it open, and only then can it
cross into the right room and reach the goal. The six actions are the four
moves, a pickup, and a toggle; the reward is $+1$ on reaching the goal and
$-0.01$ per step, the discount factor is $\gamma = 0.99$, and the horizon is
$120$. Reaching the goal terminates the episode, which removes from the
reachable set those layouts in which the goal sits behind itself, and the exact
enumeration yields $16{,}944$ reachable states, a set small enough for the
linear solves of Proposition \ref{prop:certificate} and large enough that
distillation is not trivial. The maximum absolute one-step reward is
$R_{\max} = 0.99$.

The teacher is a proximal policy optimization agent
\citep{schulman2017proximal} with a two-layer tanh actor-critic, eight
synchronous environments, and standard hyperparameters, trained on the flat
one-hot observation of the state. We study two teacher regimes over fifteen seeds
each. The converged regime, denoted R1, trains up to $800$k environment steps
with an early stop when the exact success of the greedy policy reaches $0.95$,
which every seed attains between $369$k and $451$k steps. The capped regime,
denoted R2, fixes a budget of $299{,}008$ steps, chosen so that the greedy
policy is left visibly short of convergence, which is the regime in which
expansion has room to improve on the teacher. All thirty runs use the same
environment, pinned by the SHA-256 hash of its source file, and the same
package versions, and the full learning-telemetry panel is written to disk per
update; an audit of all thirty traces with the accompanying monitor reports no
explained-variance collapse, no approximate-KL spike, no entropy collapse, and
no excessive clip fraction, so the teachers are healthy and the comparison is
between converged or deliberately-undertrained networks and their distillates,
not between broken training runs.

All headline quantities are exact. For each policy, the stochastic teacher, the
greedy teacher, the distilled program, and the expanded program, the exact
return $J$ and the exact success probability within the horizon are obtained by
the linear solves of Proposition \ref{prop:certificate}, so no Monte Carlo
noise enters the reported numbers. We aggregate across the fifteen seeds with an
own implementation of the interquartile mean and the stratified bootstrap
confidence interval in the style of \citet{agarwal2021deep}, using ten thousand
bootstrap resamples, and we report per-seed paired differences with a two-sided
exact Wilcoxon signed-rank test, correcting for the six paired tests within a
regime by the Bonferroni rule.

Figure \ref{fig:curves} shows why the two teacher regimes differ and why the
distinction matters for a distillation study. During training the stochastic
policy solves the task early, but the greedy policy that distillation actually
targets lags far behind, and in the capped regime the greedy policy is still
well short of the stochastic one at the budget. Distilling the greedy policy of
a capped teacher therefore starts from an imperfect target, which is the
headroom the expansion stage exploits.

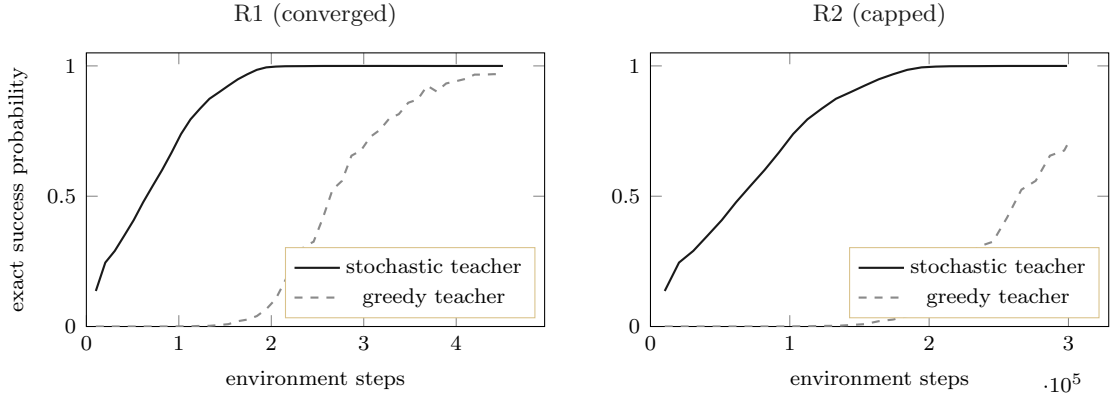
\begin{figure}[t]
\centering
\begin{tikzpicture}
\begin{axis}[
  name=ax1, width=0.48\textwidth, height=5.2cm,
  xlabel={environment steps}, ylabel={exact success probability},
  title={\footnotesize R1 (converged)}, title style={color=cInk},
  xmin=0, ymin=0, ymax=1.05, xtick scale label code/.code={},
  scaled x ticks=true, legend style={font=\scriptsize, at={(0.98,0.03)},
  anchor=south east, draw=cRule}, tick label style={font=\scriptsize},
  label style={font=\scriptsize}, every axis plot/.append style={thick}]
\addplot[cInk] table[x=step, y=success_stoch]{\curvesRone};
\addlegendentry{stochastic teacher}
\addplot[cMute, dashed] table[x=step, y=success_argmax]{\curvesRone};
\addlegendentry{greedy teacher}
\end{axis}
\begin{axis}[
  name=ax2, at={(ax1.south east)}, anchor=south west,
  xshift=1.4cm, width=0.48\textwidth, height=5.2cm,
  xlabel={environment steps}, ylabel={}, title={\footnotesize R2 (capped)},
  title style={color=cInk}, xmin=0, ymin=0, ymax=1.05,
  scaled x ticks=true, legend style={font=\scriptsize, at={(0.98,0.03)},
  anchor=south east, draw=cRule}, tick label style={font=\scriptsize},
  label style={font=\scriptsize}, every axis plot/.append style={thick}]
\addplot[cInk] table[x=step, y=success_stoch]{\curvesRtwo};
\addlegendentry{stochastic teacher}
\addplot[cMute, dashed] table[x=step, y=success_argmax]{\curvesRtwo};
\addlegendentry{greedy teacher}
\end{axis}
\end{tikzpicture}
\caption{Exact success of the stochastic and greedy PPO policies during
training, interquartile mean over fifteen seeds. The greedy policy, the target of
distillation, trails the stochastic policy throughout, and in the capped regime
R2 it is left far short of it, creating the headroom that expansion exploits.}
\label{fig:curves}
\end{figure}

The endpoint comparison is the core result and is shown in Figure
\ref{fig:endpoint}. In the converged regime the distilled program already
attains exact return $0.8122$ and exact success $1.0$ in every seed with six
clauses, and the expansion loop accepts no edit in any seed, so the expanded
and distilled programs coincide. The interesting fact about R1 is that plain
distillation already exceeds both teachers: against the stochastic teacher the
mean paired improvement in exact return is $0.0379$ with a $95\%$ confidence
interval of $[0.0354, 0.0406]$ over fifteen of fifteen seeds, because the induced
rule list generalizes the greedy action over the census and repairs the states
on which the network is imperfect. In the capped regime the distilled program is
markedly weaker and more variable, with an interquartile mean return of
$0.6032$ and a confidence interval of $[0.4977, 0.7258]$, and it is worse than
the stochastic teacher in ten of the fifteen seeds, which is the expected
consequence of distilling an undertrained greedy target. The expanded program
removes this deficit almost entirely. In every seed of both regimes the expanded
Prolog program attains exact success $1.0$, it reaches the exact optimum
$J = 0.812204$ in all seeds of R1 and in fourteen of the fifteen seeds of R2,
falling one edit short at $0.8108$ on the remaining capped seed, and in the
capped regime it exceeds the stochastic teacher on exact return in fifteen of
fifteen seeds, with a mean paired improvement of $0.1009$ and a confidence
interval of $[0.0976, 0.1042]$, a two-sided exact Wilcoxon $p = 0.0001$ and a
Bonferroni-corrected $p = 0.0004$. Against the greedy teacher that distillation
targets, the capped-regime improvement is larger still, a mean paired gain in
return of $0.5705$ with interval $[0.5081, 0.6314]$ and in success of $0.3035$
with interval $[0.2694, 0.3373]$.

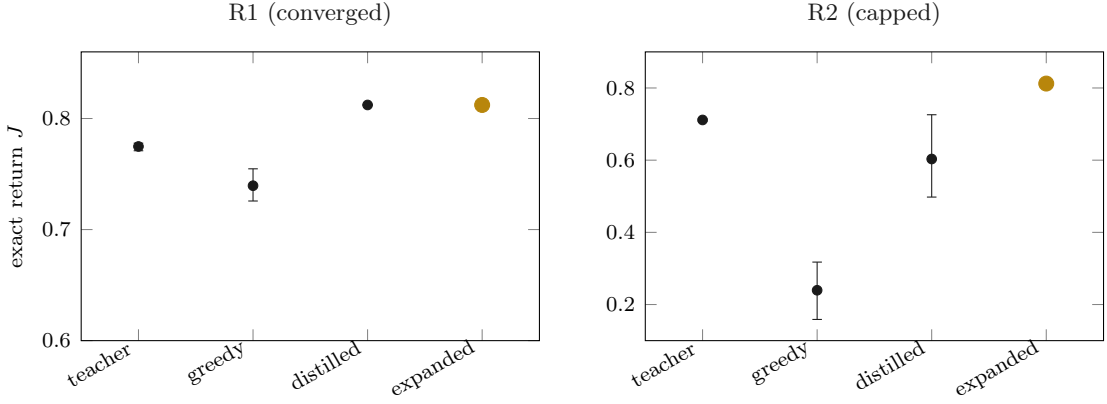
\begin{figure}[t]
\centering
\begin{tikzpicture}
\begin{axis}[
  name=e1, width=0.48\textwidth, height=5.4cm,
  ylabel={exact return $J$}, title={\footnotesize R1 (converged)},
  title style={color=cInk},
  xtick={1,2,3,4}, xticklabels={teacher, greedy, distilled, expanded},
  x tick label style={rotate=30, anchor=east, font=\scriptsize},
  ymin=0.6, ymax=0.86, tick label style={font=\scriptsize},
  label style={font=\scriptsize}, xmin=0.5, xmax=4.5]
\addplot[only marks, mark=*, cInk, mark size=1.8pt, mark options={fill=cInk},
  error bars/.cd, y dir=both, y explicit]
  table[x=x, y=J_iqm, y error plus=err_hi, y error minus=err_lo]{\endpointRoneBase};
\addplot[only marks, mark=*, cGold, mark size=2.8pt, mark options={fill=cGold}]
  table[x=x, y=J_iqm]{\endpointRoneHero};
\end{axis}
\begin{axis}[
  name=e2, at={(e1.south east)}, anchor=south west,
  xshift=1.4cm, width=0.48\textwidth, height=5.4cm,
  ylabel={}, title={\footnotesize R2 (capped)}, title style={color=cInk},
  xtick={1,2,3,4}, xticklabels={teacher, greedy, distilled, expanded},
  x tick label style={rotate=30, anchor=east, font=\scriptsize},
  ymin=0.1, ymax=0.9, tick label style={font=\scriptsize},
  label style={font=\scriptsize}, xmin=0.5, xmax=4.5]
\addplot[only marks, mark=*, cInk, mark size=1.8pt, mark options={fill=cInk},
  error bars/.cd, y dir=both, y explicit]
  table[x=x, y=J_iqm, y error plus=err_hi, y error minus=err_lo]{\endpointRtwoBase};
\addplot[only marks, mark=*, cGold, mark size=2.8pt, mark options={fill=cGold}]
  table[x=x, y=J_iqm]{\endpointRtwoHero};
\end{axis}
\end{tikzpicture}
\caption{Exact return by policy class, interquartile mean with $95\%$
stratified bootstrap confidence intervals over fifteen seeds. The expanded
student, in gold, reaches the exact optimum $0.812204$ in every seed of R1 and in
all but one seed of R2, with an interquartile-mean interval that is degenerate at
the optimum. In the capped regime R2 it exceeds the stochastic teacher and
dominates the greedy teacher and the distilled program, whose interval is wide
because capped-teacher distillation is seed-sensitive.}
\label{fig:endpoint}
\end{figure}

Two qualifications belong here rather than in a footnote. First, the expansion
effect is now statistically resolved at fifteen seeds where it was not at ten.
The expanded program strictly improves the distilled program in ten of the
fifteen capped-regime seeds, is tied on the remaining five, and is never worse,
with a mean paired improvement of $0.1985$ and a bootstrap interval of
$[0.1249, 0.2814]$ that excludes zero, a two-sided exact Wilcoxon $p = 0.0020$
that survives the six-test Bonferroni correction at $p = 0.0117$. Figure
\ref{fig:expansion} shows the underlying trajectories: each accepted edit raises
the exact return, and every seed reaches success $1.0$ after at most two edits,
the monotone behaviour that Proposition \ref{prop:expansion} guarantees. Second,
expansion is not a fidelity-improving operation. The expanded program's fidelity
to the stochastic teacher under rollout visitation is about $0.71$ in the capped
regime, and its disagreement rate against the greedy teacher can rise after
expansion, because accepted edits deliberately move the student away from an
imperfect teacher. Fidelity is the accounting quantity of Theorem
\ref{thm:distillation}, not the objective of Proposition \ref{prop:expansion},
and the student is best understood as a repaired program rather than a faithful
clone.

\begin{figure}[t]
\centering
\begin{tikzpicture}
\begin{axis}[
  width=0.62\textwidth, height=5.2cm,
  xlabel={accepted edit index}, ylabel={exact return $J$},
  xtick={0,1,2}, ymin=0.3, ymax=0.86, xmin=-0.1, xmax=2.1,
  tick label style={font=\scriptsize}, label style={font=\scriptsize}]
\pgfplotsinvokeforeach{0,...,14}{
  \addplot[cMute, thick, mark=*, mark size=1pt, opacity=0.75]
    table[x=edit_index, y=seed#1]{\expansionRtwo};
}
\addplot[cGold, dashed, thick, samples=2, domain=-0.1:2.1] {0.7114};
\node[font=\scriptsize, color=cGoldDeep, anchor=south east]
  at (axis cs:2.1,0.7117) {stochastic teacher $J$};
\end{axis}
\end{tikzpicture}
\caption{Capped-regime expansion trajectories, exact return after each accepted
edit, one grey line per seed. Every accepted edit increases the exact return, as
Proposition \ref{prop:expansion} requires, and all fifteen seeds cross above the
stochastic teacher's return, shown as the gold dashed reference, after at most
two edits, fourteen of them reaching the common optimum $0.812204$ and one
halting at $0.8108$.}
\label{fig:expansion}
\end{figure}
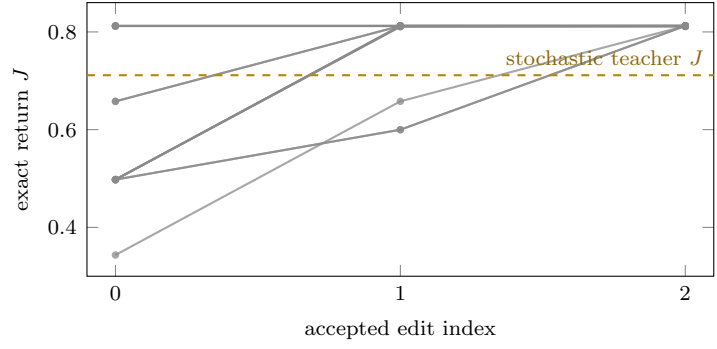

The two certificates of Section~\ref{sec:theory} were evaluated on all sixty
combinations of regime, seed, and student, and the contrast between them is the
point. The worst-case inequality of Theorem~\ref{thm:distillation} holds in
every record but vacuously: with the prefactor $2 R_{\max} / (1-\gamma)^2$ of
order $2 \times 10^4$ at $\gamma = 0.99$, its interquartile-mean right-hand side
is about $9.9 \times 10^2$ on the converged regime and about $6.2 \times 10^3$ on
the capped regime, whereas the exact return gap it purports to bound is of order
$10^{-1}$. The advantage-gap certificate of
Proposition~\ref{prop:advantage-certificate} closes this chasm. On the converged
regime it is tighter than the Theorem~\ref{thm:distillation} bound by a median
factor of about $13{,}700$, its refined right-hand side
$\bar{A}^{\dagger} / (1-\gamma)$ has interquartile mean $0.0727$ against an exact
gap of the same $0.0727$, and it falls below unity in every one of the fifteen
seeds. The reason it coincides with the exact gap is the sign condition of
Corollary~\ref{cor:nonvacuity}: in all fifteen converged-regime seeds every
disagreement advantage $A^{\pi_S}(s, a_T(s))$ is non-positive, the student weakly
dominating the argmax teacher at every state where they differ, so the triangle
inequality in the proof is an equality and the certificate is exactly
non-vacuous. The capped regime tells the same story for the expanded program,
whose refined bound $0.5726$ equals its exact gap $0.5726$ with the sign
condition met in fourteen of fifteen seeds and the bound below unity in all
fifteen, and a weaker story for the distilled program, where the sign condition
holds in only five seeds, so the refined bound $0.7385$ overshoots the exact gap
$0.3658$ through genuine slack yet still lands below unity in thirteen of
fifteen seeds. Table~\ref{tab:certificate} and Figure~\ref{fig:certificate}
report both certificates side by side; the message is that a horizon-aware,
advantage-localized bound turns a guarantee that was slack by four orders of
magnitude into one that is checked, non-vacuous, and on the regime the
distillation targets, exact.

\begin{table}[t]
\centering
\caption{The two certificates over all sixty (regime, seed, student) records,
aggregated to the interquartile mean per regime and student. The worst-case
bound of Theorem~\ref{thm:distillation} is vacuous; the advantage-gap bound
$\bar{A}^{\dagger} / (1-\gamma)$ of Proposition~\ref{prop:advantage-certificate}
is tighter by the median factor shown and coincides with the exact return gap
$|J(\bar{\pi}_T) - J(\pi_S)|$ whenever the disagreement advantages are
non-positive, the condition counted in the last column. On the converged regime
R1 the expansion loop accepts no edit, so the distilled and expanded programs
coincide.}
\label{tab:certificate}
\small
\setlength{\tabcolsep}{4.5pt}
\begin{tabular}{llccccc}
\toprule
regime & student & $\epsilon^{\dagger}$ &
Thm.~\ref{thm:distillation} bound & $\bar{A}^{\dagger} / (1-\gamma)$ &
exact $|{\rm gap}|$ & non-vacuous seeds \\
\midrule
R1 & distilled $=$ expanded & $0.050$ & $9.94\times10^2$ & $0.073$ & $0.073$ &
$15/15$ \\
R2 & distilled & $0.292$ & $5.78\times10^3$ & $0.739$ & $0.366$ & $5/15$ \\
R2 & expanded & $0.315$ & $6.24\times10^3$ & $0.573$ & $0.573$ & $14/15$ \\
\bottomrule
\end{tabular}
\end{table}

\begin{figure}[t]
\centering
\begin{tikzpicture}
\begin{axis}[
  width=0.72\textwidth, height=5.6cm,
  ybar, bar width=6pt, ymode=log, ymin=0.03, ymax=3e4,
  symbolic x coords={R1, R2 dist., R2 exp.},
  xtick=data, enlarge x limits=0.28,
  ylabel={return-gap bound (log scale)}, ylabel style={font=\scriptsize},
  tick label style={font=\scriptsize},
  legend style={font=\scriptsize, at={(0.02,0.98)}, anchor=north west,
    draw=cRule, legend columns=1},
  nodes near coords style={font=\tiny}]
\addplot[fill=cInk, draw=cInk]
  table[x=cat, y=old_thm1, col sep=comma]{data/certificate_tightening_plot.csv};
\addlegendentry{Thm.~1 worst-case bound}
\addplot[fill=cGold, draw=cGoldDeep]
  table[x=cat, y=refined, col sep=comma]{data/certificate_tightening_plot.csv};
\addlegendentry{advantage-gap bound $\bar{A}^{\dagger}/(1-\gamma)$}
\addplot[fill=cMute, draw=cMute]
  table[x=cat, y=exact_gap, col sep=comma]{data/certificate_tightening_plot.csv};
\addlegendentry{exact return gap}
\end{axis}
\end{tikzpicture}
\caption{The two certificates and the exact return gap, interquartile mean over
fifteen seeds, on a logarithmic scale. The Theorem~\ref{thm:distillation} bound
in ink sits four orders of magnitude above the exact gap; the advantage-gap
bound in gold sits on or just above it. On R1 and on the R2 expanded program the
gold and grey bars coincide because every disagreement advantage is non-positive
and the certificate is exactly non-vacuous; on the R2 distilled program the gold
bar exceeds the grey by the slack the sign condition leaves.}
\label{fig:certificate}
\end{figure}
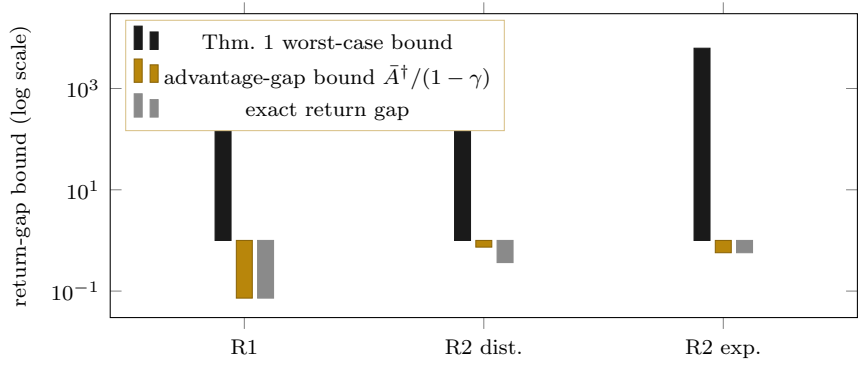

The artefact the pipeline produces is a program one can read. Listing
\ref{lst:prolog} shows the decision list of the expanded program for the first
capped-regime seed, verbatim from the emitted file. The rules read as a
strategy: follow the path to the goal if one exists, otherwise pick up the key
when standing on it, otherwise walk to the key, otherwise toggle the door when
adjacent and carrying the key, and so on, with the path-aware direction
predicate hiding the navigation in a single tabled breadth-first search. The
first clause is subtler than it looks, because the direction predicate to the
goal succeeds only when a path exists, so the clause fires only once the door
is open and otherwise falls through to the key-fetching logic. The full program,
including the perception layer and the breadth-first search, is $133$ lines and
is machine-verified to select the same action as the reference evaluator on
every state at which it was queried.

\begin{lstlisting}[style=prolog, caption={The decision list of the expanded
Prolog program for capped-regime seed $0$, copied verbatim from the emitted
file. The perception layer and the tabled path-aware direction predicate
\texttt{dir\_to/3}, together $133$ lines, are omitted here and run unchanged
under SWI-Prolog.}, label={lst:prolog}]
% clause 0: act = move(D) :- dir_to(goal, D).
act(S, move(D)) :- dir_to(S, goal, D), !.
% clause 1: act = pickup :- on_key, not door_open, not same_room_goal.
act(S, pickup) :- on_key(S), \+ door_open(S), \+ same_room_goal(S), !.
% clause 2: act = move(D) :- dir_to(key, D).
act(S, move(D)) :- dir_to(S, key, D), !.
% clause 3: act = toggle :- adj_door, carrying_key.
act(S, toggle) :- adj_door(S), carrying(S), !.
% clause 4: act = move(D) :- dir_to(door, D), carrying_key.
act(S, move(D)) :- dir_to(S, door, D), carrying(S), !.
% clause 5: act = pickup :- not adj_door.
act(S, pickup) :- \+ adj_door(S), !.
% clause 6 (default): act = up.
act(_S, move(up)).
% clause 7 (default): act = right.
act(_S, move(right)).
\end{lstlisting}

\label{sec:minigrid}

The key-and-door study shows that a first-order program can be exact and
certified, but its predicate vocabulary, path-aware direction and adjacency, does
work that a propositional list could in principle imitate on that single fixed
layout. The value of a relational representation is that one program covers a
family of layouts, and the sharpest test of that value is transfer to a size the
student never saw. We run this test on the DoorKey task of the MiniGrid suite
\citep{chevalier2023minigrid}, in which an agent on an $N \times N$ grid must
reach a key, carry it to a locked door, toggle the door open, and reach a goal in
the far room, receiving a shaped terminal reward that decays with the number of
steps taken. A proximal policy optimization teacher \citep{schulman2017proximal}
is trained on the $8 \times 8$ layout over fifteen seeds, and from each teacher we
distil, on the $8 \times 8$ layout alone, five students that are then evaluated
without any further training on the $6 \times 6$, $8 \times 8$, and
$16 \times 16$ layouts. The student we advocate is the first-order Prolog list of
Definition~\ref{def:dlist} over the relational vocabulary of
Section~\ref{sec:algorithm}, and it is compared against four decision-tree
students that isolate the one variable that matters. Two are CART trees
\citep{breiman1984classification} and two are VIPER trees
\citep{bastani2018verifiable}, and within each pair one tree is fitted over the
same agent-relative relational features the Prolog student uses and the other
over the absolute cell coordinates of the objects and the agent. The comparison
is therefore not tree against logic program but relational representation against
propositional representation, holding the induction algorithm fixed.

Figure~\ref{fig:minigrid} is the central result of the paper. On the
$8 \times 8$ layout it was distilled on, every student that carries a workable
representation reaches the teacher's return: the Prolog list scores an
interquartile-mean return of $0.950$, the relational CART and VIPER trees $0.968$
and $0.969$, and the teacher itself $0.966$, while even the coordinate trees
manage a partial $0.612$ and $0.505$ by memorizing the training geometry. Off the
training size the two representations part completely. The Prolog list transfers
to the smaller $6 \times 6$ grid at return $0.947$ and to the larger
$16 \times 16$ grid at $0.988$, matching the teacher's own $0.950$ and $0.978$
and reaching a perfect success rate of $1.000$ on the large grid; the relational
trees transfer just as well, at $0.948$ and $0.968$ for CART and $0.947$ and
$0.979$ for VIPER. The coordinate trees, by contrast, collapse to exactly zero
return and zero success at both unseen sizes, because a split of the form
``agent column $\ge 5$'' or ``goal row $< 4$'' encodes an absolute position that
no longer denotes the same situation once the grid is resized, and the tree that
memorized the $8 \times 8$ geometry now fires the wrong action at every state. The
paired contrast is unambiguous: the Prolog list beats the coordinate CART tree by
an interquartile-mean return margin of $0.947$ at $6 \times 6$ and $0.988$ at
$16 \times 16$, in fifteen of fifteen seeds at each size, with a two-sided exact
Wilcoxon $p = 6.1 \times 10^{-5}$ that survives Holm correction over the
eight-contrast family at $p = 4.9 \times 10^{-4}$, and identically against the
coordinate VIPER tree.

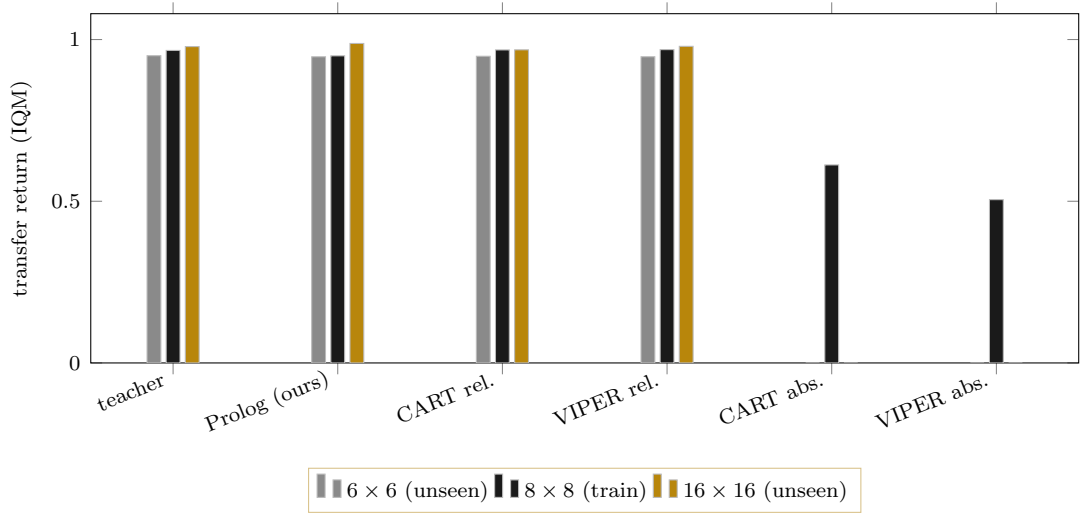
\begin{figure}[t]
\centering
\begin{tikzpicture}
\begin{axis}[
  width=0.92\textwidth, height=6.2cm,
  ybar, bar width=5.2pt, ymin=0, ymax=1.08,
  symbolic x coords={teacher, {Prolog (ours)}, {CART rel.}, {VIPER rel.},
    {CART abs.}, {VIPER abs.}},
  xtick=data, enlarge x limits=0.10,
  x tick label style={rotate=22, anchor=east, font=\scriptsize},
  ylabel={transfer return (IQM)}, ylabel style={font=\scriptsize},
  tick label style={font=\scriptsize},
  legend style={font=\scriptsize, at={(0.5,-0.30)}, anchor=north,
    legend columns=3, draw=cRule},
  every axis plot/.append style={draw=cInk!30}]
\addplot[fill=cMute]
  table[x=student, y=ret6, col sep=comma]{data/minigrid_transfer.csv};
\addlegendentry{$6\times6$ (unseen)}
\addplot[fill=cInk]
  table[x=student, y=ret8, col sep=comma]{data/minigrid_transfer.csv};
\addlegendentry{$8\times8$ (train)}
\addplot[fill=cGold]
  table[x=student, y=ret16, col sep=comma]{data/minigrid_transfer.csv};
\addlegendentry{$16\times16$ (unseen)}
\end{axis}
\end{tikzpicture}
\caption{Cross-size transfer on MiniGrid DoorKey, interquartile-mean return over
fifteen seeds. Every student is distilled on the $8 \times 8$ layout and
evaluated on all three sizes without retraining. The first-order Prolog student
and the two decision trees over agent-relative relational features transfer at
the teacher's level to the unseen $6 \times 6$ and $16 \times 16$ grids; the two
decision trees over absolute cell coordinates match on the training size but
collapse to zero return at every other size, the gap the paper turns on. Bars are
grouped by student; within each group the training size is the ink bar and the
two unseen sizes flank it.}
\label{fig:minigrid}
\end{figure}

The relational student pays for this transfer in neither size nor readability,
which is the second half of the result. The Prolog list induced on $8 \times 8$
has eight clauses over twelve literals for the first seed and a median of nine
clauses across the fifteen seeds, and it reads as a strategy a person could have
written: navigate to the goal if a path exists, otherwise pick up the key when
facing it, otherwise toggle the door when facing it and carrying the key,
otherwise navigate toward the key and then the door, with a single tabled
breadth-first search hidden inside the navigation predicate. Listing
\ref{lst:minigrid} shows the decision list verbatim. The relational trees are of
comparable size, a median of seventeen and fourteen leaves for CART and VIPER,
but the coordinate trees that fail to transfer are an order of magnitude larger,
a median of $359$ and $387$ leaves, because absolute coordinates force the tree
to enumerate positions rather than name relations; the teacher network they all
distil carries $12{,}038$ parameters. The one quantity on which the Prolog
student does not lead is held-out fidelity on the training size, where it agrees
with the teacher on an interquartile-mean $0.907$ of states against about $0.987$
for the relational trees, yet it transfers exactly as well, because the
path-aware navigation predicate reproduces the teacher's shortest-path behaviour
structurally rather than by matching its action on each individual state. On this
task, fidelity on the training distribution is the wrong figure of merit and
size-invariant structure is the right one.

\begin{lstlisting}[style=prolog, caption={The MiniGrid DoorKey decision list
induced on the $8 \times 8$ layout for seed $0$, copied verbatim from the emitted
Prolog file; the perception predicates and the tabled shortest-path predicate
\texttt{nav/3} are defined above it and run unchanged on grids of any size. Action
codes follow MiniGrid: $1$ turn right, $3$ pick up, $4$ toggle; \texttt{nav/3}
returns the turn-or-forward action that advances along a shortest path to the
named object.}, label={lst:minigrid}]
% clause 0: navigate to the goal if a path to it exists
act(S, A)  :- nav(S, goal, A), !.
% clause 1: pick up the key when facing it
act(S, 3)  :- facing_key(S), !.
% clause 2: toggle the door when facing it and carrying the key
act(S, 4)  :- facing_door(S), carrying_key(S), !.
% clause 3: toggle the door when facing it
act(S, 4)  :- facing_door(S), !.
% clause 4: otherwise navigate toward the key
act(S, A)  :- nav(S, key, A), !.
% clause 5: otherwise navigate toward the door
act(S, A)  :- nav(S, door, A), !.
% clause 6: turn right to reorient when the cell ahead is blocked
act(S, 1)  :- \+ facing_clear(S), !.
% clause 7 (default): toggle
act(_S, 4).
\end{lstlisting}

\label{sec:continuous}

The DoorKey transfer shows what a relational representation buys. The
continuous-control study shows, just as deliberately, what happens when there is
no relational structure to represent. We run the same pipeline, in its
propositional instantiation, on continuous-observation control tasks from the
Gymnasium suite \citep{towers2024gymnasium}: CartPole with a four-dimensional
observation and two actions and Acrobot with a six-dimensional observation and
three actions, each over fifteen teacher seeds, and, as the high-dimensional
reference point, LunarLander with an eight-dimensional observation and four
actions over five seeds. In a four to six dimensional vector of physical
quantities there are no objects to quantify over and no relations to bind, so the
first-order machinery of Definition~\ref{def:dlist} collapses to an ordered list
of axis-aligned threshold conjunctions, and the question is no longer transfer
but how much of the network a propositional rule list can recover on its own
training distribution and how it compares to the propositional baselines built
for exactly this job. Two things change relative to KeyDoor, and both are the
substitutions anticipated in the discussion of scale. First, the reachable-state
census is replaced by the DAgger reduction of imitation to no-regret online
learning \citep{ross2011reduction}, the same loop VIPER uses to extract a tree
student \citep{bastani2018verifiable}: the student is rolled out, the states it
visits are relabelled by the frozen teacher, and the union of teacher-visited and
student-visited states drives the next induction, so the student is trained
exactly on the off-distribution states its own errors produce. Second, exact
policy evaluation is replaced by Monte-Carlo evaluation over a fixed episode-seed
stream shared by the teacher and every student, so all comparisons within a task
are paired at the level of initial conditions and every reported return carries a
$95\%$ percentile-bootstrap confidence interval over episodes.

The student class is the propositional specialization of Definition
\ref{def:dlist}. Each continuous feature is discretized by a data-driven grid of
thresholds placed at the empirical quantiles of the teacher's visited values,
controlled by a single resolution parameter $B$ that fixes the number of bins per
feature; a base predicate is a threshold test $\mathrm{feat} \ge t$, its negation
is the exact complement $\mathrm{feat} < t$, and the induction of Section
\ref{sec:algorithm} grows an ordered list of clauses over these threshold
literals exactly as before. The resolution $B$ interpolates between readability
and precision: at $B=2$ a single median threshold per feature yields a handful of
broad clauses, and as $B$ grows the grid carves the state space into finer
axis-aligned regions each labelled by an action, which is the propositional form
of the partition that replaces the network. We sweep
$B \in \{2, 3, 4, 6, 8, 12\}$ over fifteen PPO teacher seeds per task on CartPole
and Acrobot, and over five seeds on LunarLander, emit every student as an
executable SWI-Prolog program whose action is asserted to match the reference
evaluator on at least five hundred sampled states, and aggregate across seeds
with the interquartile mean and a stratified bootstrap in the manner of
\citet{agarwal2021deep}. At fifteen seeds the two-sided exact Wilcoxon
signed-rank test no longer floors, so on CartPole and Acrobot we report corrected
$p$-values alongside effect sizes; on the five-seed LunarLander reference the test
floors at $p = 0.0625$ and the claims rest on effect sizes and intervals. All
thirty CartPole and Acrobot teachers solve their task, with the greedy policy
that distillation targets clearing the solved threshold on the shared evaluation
stream, and the learning-telemetry audit is clean on the approximate KL, the
entropy, and the clip fraction throughout. The one exception is that explained
variance stays low on the fast-solving CartPole seeds, which is the expected
consequence of a near-constant value target once the pole is balanced rather than
a training pathology, and it does not bear on the distilled student because the
target of distillation is the greedy action, not the critic; every teacher's
success on the evaluation stream confirms that a healthy network is being
distilled.

Figure \ref{fig:continuous} shows the first finding, that the recovered symbolic
program substitutes the network at a task-dependent resolution, and that the
three tasks form a gradient in how hard the substitution is. On LunarLander the
naive one-pass distillation fails at every resolution, flying the lander into the
ground with a return between $-467$ and $-240$ and a success rate near zero even
as its held-out agreement with the teacher rises with $B$, which is the textbook
signature of behavioural-cloning distribution shift: agreement on the teacher's
own states does not prevent compounding error once the coarse student drifts off
that distribution. DAgger repairs this monotonically, climbing from a return of
$-201$ at $B=2$ to $+155$ at $B=12$ and turning positive from $B=6$, but it does
not reach the teacher's $221$, because the propositional threshold list has a
genuine ceiling on the fine continuous control that a soft landing demands.
CartPole and Acrobot tell the complementary story. On CartPole, whose teacher
balances for nearly the full horizon at an interquartile-mean return of $489.8$
against a solved threshold of $475$, the DAgger program recovers about $98\%$ of
that return, reaching $473.3$ at $B=3$ with roughly thirty-four clauses and
$479.6$ at $B=12$ where its distance to the teacher is $-7.5$ with a confidence
interval of $[-23.3, -1.4]$, and it improves on the naive student at every
resolution with a paired-difference interval that excludes zero from $B=4$ upward.
On Acrobot, whose teacher raises the tip to an interquartile-mean return of
$-84.3$ against a solved threshold of $-100$, the DAgger program matches the
network within noise at the fine resolutions, with a return of $-85.1$ at $B=8$
and $-84.7$ at $B=12$ whose distances to the teacher, $-1.0$ with interval
$[-2.1, +0.3]$ and $-0.5$ with interval $[-2.2, +0.8]$, both bracket zero, and
even the coarsest readable list at $B=2$, some twelve clauses, reaches $-87.2$
with a $0.87$ success rate. The distribution-shift gradient is visible across the
three panels: on Acrobot, a locally self-correcting swing-up task, the naive and
DAgger curves nearly coincide; on CartPole DAgger clearly helps; on LunarLander
DAgger is indispensable. Both CartPole and Acrobot exhibit a single non-monotone
dip at an intermediate resolution, $B=4$ for the former and $B=3$ for the latter,
where the quantile grid places its thresholds across the decision boundary
particularly poorly and the return recovers at finer resolution; we report the
dip rather than smooth it. These curves are the empirical face of the theory of
Section~\ref{sec:theory}: Corollary~\ref{cor:resolution-return} guarantees that
the return gap closes as the resolution grows, and
Proposition~\ref{prop:resolution-lb} explains why an eight-dimensional
observation makes that closure expensive enough to leave a visible gap on
LunarLander at the resolutions we sweep, while the four-dimensional CartPole and
the six-dimensional Acrobot, with a boundary of lower codimension to carve, close
it within the same budget.

Substituting the network in return is not the same as matching the interpretable
students built for propositional control, and the honest comparison against them
is the second finding, reported in Table~\ref{tab:baselines} and
Figure~\ref{fig:baselines}. On the identical fifteen-seed grid and the identical
paired evaluation stream we place the DAgger decision list against a CART tree
grown on the raw observation \citep{breiman1984classification}, against a VIPER
tree that adds importance resampling to the same induction
\citep{bastani2018verifiable}, and against a Coppens-style unordered rule set that
votes over the same threshold vocabulary as our list \citep{coppens2021rule}, all
in the DAgger regime, with the return delta averaged over the resolution ladder
per seed and aggregated as an interquartile mean with a Holm-corrected two-sided
Wilcoxon test over the per-environment comparison family. The result is a clean
split. Our decision list loses to CART and to VIPER on CartPole, by an
interquartile-mean margin of $78.5$ and $76.2$ points of return, in zero of the
fifteen seeds does it come out ahead, and the Holm-corrected $p$ is
$3.7 \times 10^{-4}$; on Acrobot the same two trees are ahead by a smaller
$15.0$ and $14.0$ points whose intervals place the result at the boundary between
a loss and a statistical tie. Against the propositional Coppens rule set our list
wins on both tasks, by $59.0$ points on CartPole and $62.6$ on Acrobot with
probability of improvement near one, so it is the strongest ordered rule list
among the logic-style baselines while remaining behind the trees. This deficit is
not an artifact of the data-driven threshold grid: replacing the quantile cuts by
the Gini-optimal one-dimensional split points that a tree would choose moves the
return by a mean of $-47.6$ points on CartPole and $+19.0$ on Acrobot and closes
none of the gap to the tree baselines on average, so the shortfall is
representational rather than a matter of where the cuts are placed. We read it as
expected. A CART or VIPER tree splits the joint feature space with oblique
sequences of axis tests and reuses each internal node across many leaves, whereas
our first-match ordered list commits to one feature at a time and cannot share
structure across clauses, and on a low-dimensional vector of physical quantities,
where there is no relational regularity for the first-order form to capture, that
sharing is exactly what a soft-margin controller needs. The representation that
wins the DoorKey transfer is the one that loses here, and both outcomes have the
same cause.

\begin{figure}[t]
\centering
\begin{tikzpicture}
\begin{axis}[
  name=cp1, width=0.35\textwidth, height=4.7cm,
  xlabel={resolution $B$}, ylabel={Monte-Carlo return},
  title={\footnotesize LunarLander}, title style={color=cInk},
  xmin=1.5, xmax=12.5, xtick={2,4,6,8,12},
  tick label style={font=\scriptsize}, label style={font=\scriptsize},
  legend style={font=\tiny, at={(0.02,0.98)}, anchor=north west, draw=cRule},
  every axis plot/.append style={thick}]
\addplot[cMute, dashed, mark=square*, mark size=1.3pt]
  table[x=B, y=naive, col sep=comma]{data/lunar_return.csv};
\addlegendentry{naive}
\addplot[cGold, mark=*, mark size=1.6pt, error bars/.cd, y dir=both, y explicit]
  table[x=B, y=dagger, y error plus=derr_hi, y error minus=derr_lo, col sep=comma]{data/lunar_return.csv};
\addlegendentry{DAgger}
\addplot[cInk, densely dotted] table[x=B, y=teacher, col sep=comma]{data/lunar_return.csv};
\addlegendentry{teacher}
\end{axis}
\begin{axis}[
  name=cp2, at={(cp1.south east)}, anchor=south west, xshift=0.95cm,
  width=0.35\textwidth, height=4.7cm,
  xlabel={resolution $B$}, ylabel={},
  title={\footnotesize CartPole}, title style={color=cInk},
  xmin=1.5, xmax=12.5, xtick={2,4,6,8,12},
  tick label style={font=\scriptsize}, label style={font=\scriptsize},
  every axis plot/.append style={thick}]
\addplot[cMute, dashed, mark=square*, mark size=1.3pt]
  table[x=B, y=naive, col sep=comma]{data/cartpole_return.csv};
\addplot[cGold, mark=*, mark size=1.6pt, error bars/.cd, y dir=both, y explicit]
  table[x=B, y=dagger, y error plus=derr_hi, y error minus=derr_lo, col sep=comma]{data/cartpole_return.csv};
\addplot[cInk, densely dotted] table[x=B, y=teacher, col sep=comma]{data/cartpole_return.csv};
\end{axis}
\begin{axis}[
  name=cp3, at={(cp2.south east)}, anchor=south west, xshift=0.95cm,
  width=0.35\textwidth, height=4.7cm,
  xlabel={resolution $B$}, ylabel={},
  title={\footnotesize Acrobot}, title style={color=cInk},
  xmin=1.5, xmax=12.5, xtick={2,4,6,8,12},
  tick label style={font=\scriptsize}, label style={font=\scriptsize},
  every axis plot/.append style={thick}]
\addplot[cMute, dashed, mark=square*, mark size=1.3pt]
  table[x=B, y=naive, col sep=comma]{data/acrobot_return.csv};
\addplot[cGold, mark=*, mark size=1.6pt, error bars/.cd, y dir=both, y explicit]
  table[x=B, y=dagger, y error plus=derr_hi, y error minus=derr_lo, col sep=comma]{data/acrobot_return.csv};
\addplot[cInk, densely dotted] table[x=B, y=teacher, col sep=comma]{data/acrobot_return.csv};
\end{axis}
\end{tikzpicture}
\caption{Monte-Carlo return of the propositional student against resolution $B$
on three continuous-control tasks, interquartile mean with $95\%$ bootstrap
confidence intervals on the DAgger curve, over fifteen seeds for CartPole and
Acrobot and five for LunarLander. The naive one-pass student is in grey, the
DAgger student in gold, and the neural teacher's return is the ink dotted line.
The naive student fails on LunarLander at every resolution and recovers on the
more forgiving CartPole and Acrobot; the DAgger student climbs toward the teacher
on all three, reaching it within noise on Acrobot at $B \ge 8$ and recovering
about $98\%$ of it on CartPole, while retaining a genuine gap on the fine-control
LunarLander task. The intermediate-resolution dips on CartPole ($B=4$) and
Acrobot ($B=3$) are honest non-monotonicities of the quantile grid.}
\label{fig:continuous}
\end{figure}
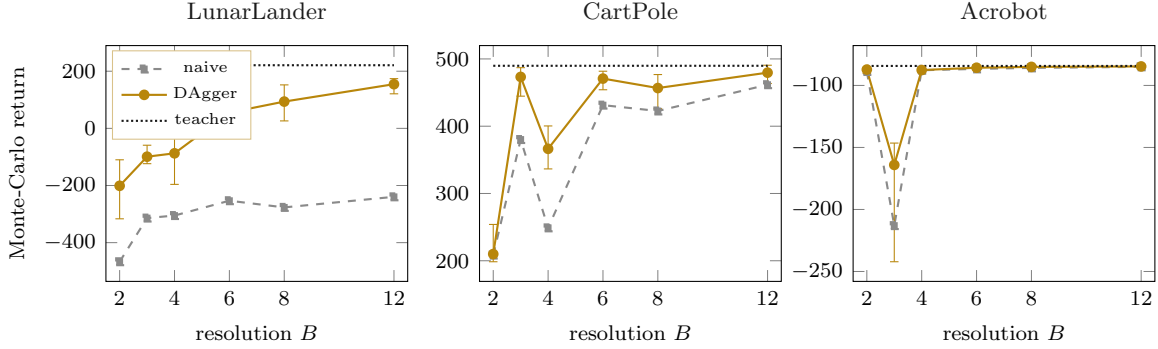

\begin{table}[t]
\centering
\caption{The DAgger decision list against the interpretable baselines on the
identical fifteen-seed grid and paired evaluation stream. The reported quantity
is the return of our list minus the return of the baseline, averaged over the
resolution ladder per seed, as an interquartile mean with a $95\%$ stratified
bootstrap interval; P(improve) is the fraction of seeds on which our list scores
higher, and the Holm-corrected two-sided exact Wilcoxon $p$ is over the
per-environment comparison family. A negative delta is a loss for our list.}
\label{tab:baselines}
\small
\setlength{\tabcolsep}{4.5pt}
\begin{tabular}{llccccl}
\toprule
task & baseline & our list $-$ baseline & P(improve) & Wilcoxon $p$ &
Holm $p$ & verdict \\
\midrule
CartPole & CART (tree) & $-78.5\ [-90.6, -66.0]$ & $0.00$ &
$6.1\times10^{-5}$ & $3.7\times10^{-4}$ & loses \\
CartPole & VIPER (tree) & $-76.2\ [-88.2, -64.3]$ & $0.00$ &
$6.1\times10^{-5}$ & $3.7\times10^{-4}$ & loses \\
CartPole & Coppens (rules) & $+59.0\ [+39.4, +79.1]$ & $0.93$ &
$1.8\times10^{-4}$ & $3.7\times10^{-4}$ & beats \\
Acrobot & CART (tree) & $-15.0\ [-28.2, -11.7]$ & $0.00$ &
$6.1\times10^{-5}$ & $3.7\times10^{-4}$ & matches \\
Acrobot & VIPER (tree) & $-14.0\ [-27.6, -10.5]$ & $0.00$ &
$6.1\times10^{-5}$ & $3.7\times10^{-4}$ & matches \\
Acrobot & Coppens (rules) & $+62.6\ [+30.9, +99.3]$ & $1.00$ &
$6.1\times10^{-5}$ & $3.7\times10^{-4}$ & beats \\
\bottomrule
\end{tabular}
\end{table}

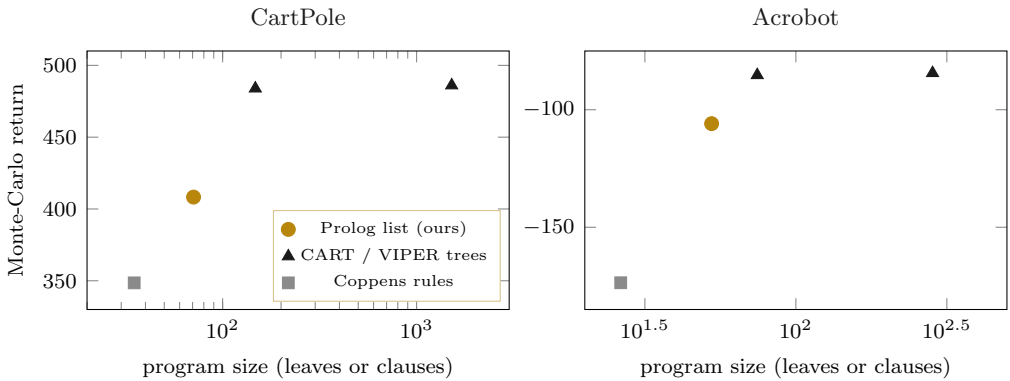
\begin{figure}[t]
\centering
\begin{tikzpicture}
\begin{axis}[
  name=bl1, width=0.45\textwidth, height=5.0cm, xmode=log,
  xlabel={program size (leaves or clauses)}, ylabel={Monte-Carlo return},
  title={\footnotesize CartPole}, title style={color=cInk},
  xmin=20, xmax=3000, ymin=330, ymax=510,
  tick label style={font=\scriptsize}, label style={font=\scriptsize},
  legend style={font=\tiny, at={(0.98,0.03)}, anchor=south east, draw=cRule}]
\addplot[only marks, mark=*, cGold, mark size=2.6pt] coordinates {(70.8,408.27)};
\addlegendentry{Prolog list (ours)}
\addplot[only marks, mark=triangle*, cInk, mark size=2.6pt]
  coordinates {(1516.4,486.05) (147.4,483.88)};
\addlegendentry{CART / VIPER trees}
\addplot[only marks, mark=square*, cMute, mark size=2.2pt]
  coordinates {(35.0,348.55)};
\addlegendentry{Coppens rules}
\end{axis}
\begin{axis}[
  name=bl2, at={(bl1.south east)}, anchor=south west, xshift=1.0cm,
  width=0.45\textwidth, height=5.0cm, xmode=log,
  xlabel={program size (leaves or clauses)}, ylabel={},
  title={\footnotesize Acrobot}, title style={color=cInk},
  xmin=20, xmax=500, ymin=-185, ymax=-75,
  tick label style={font=\scriptsize}, label style={font=\scriptsize}]
\addplot[only marks, mark=*, cGold, mark size=2.6pt] coordinates {(52.6,-105.92)};
\addplot[only marks, mark=triangle*, cInk, mark size=2.6pt]
  coordinates {(283.6,-84.40) (74.5,-85.27)};
\addplot[only marks, mark=square*, cMute, mark size=2.2pt]
  coordinates {(26.3,-173.62)};
\end{axis}
\end{tikzpicture}
\caption{Return against program size for the DAgger decision list and the
interpretable baselines, each point the mean over fifteen seeds and the six
resolutions, on a logarithmic size axis. Our first-order list, in gold, is the
most compact interpretable student but sits below the CART and VIPER trees in
ink on return; the Coppens rule set in grey is both smaller and weaker. On
propositional control the trees dominate the return-versus-size frontier, the
mirror image of the DoorKey transfer where the relational representation
dominates. The list means are pulled down by the single intermediate-resolution
dip visible in Figure~\ref{fig:continuous}.}
\label{fig:baselines}
\end{figure}

The resolution sweep is itself a direct test of Theorem~\ref{thm:resolution} on
the trained policies rather than on the synthetic boundary of
Proposition~\ref{prop:resolution-lb}, and Figure~\ref{fig:scaling} reports its
two governing quantities. The held-out fidelity, the fraction of teacher states
on which the student agrees with the network, rises monotonically with the
resolution on every task, which is the empirical face of the vanishing
disagreement $\epsilon^{\dagger}_B \to 0$ that the theorem guarantees. The three
curves are ordered by difficulty rather than strictly by dimension: CartPole and
Acrobot reach fidelity above $0.95$ by $B = 8$, whereas the eight-dimensional
LunarLander plateaus near $0.78$ at the finest resolution we sweep, the
same task on which the return gap stays open, so the fidelity ceiling and the
return ceiling are one phenomenon seen through two metrics. The clause count is
the interpretability price of that fidelity: it grows with the resolution on
every task, from nine to about a hundred and twenty on CartPole and from twenty
to about three hundred on LunarLander, so the resolution parameter is literally
the exchange rate between how finely the program partitions the state space and
how many rules a reader must hold in mind. We do not claim the induced clause
count follows the $B^{d-1}$ law of Proposition~\ref{prop:resolution-lb}, which
counts the cells a worst-case oblique boundary forces and is confirmed on that
boundary in Section~\ref{sec:theory}; the induced count is the task-dependent
number the greedy covering actually emits, and it is reported as the measured
cost, not as a test of the exponent.

\begin{figure}[t]
\centering
\begin{tikzpicture}
\begin{axis}[
  name=sc1, width=0.46\textwidth, height=5.0cm,
  xlabel={resolution $B$}, ylabel={held-out fidelity to the network},
  xmin=1.5, xmax=12.5, xtick={2,4,6,8,12}, ymin=0.4, ymax=1.0,
  tick label style={font=\scriptsize}, label style={font=\scriptsize},
  legend style={font=\tiny, at={(0.98,0.03)}, anchor=south east, draw=cRule},
  every axis plot/.append style={thick}]
\addplot[cGold, mark=*, mark size=1.5pt]
  table[x=B, y=cartpole, col sep=comma]{data/scaling_fidelity.csv};
\addlegendentry{CartPole ($d{=}4$)}
\addplot[cGoldDeep, mark=triangle*, mark size=1.8pt]
  table[x=B, y=acrobot, col sep=comma]{data/scaling_fidelity.csv};
\addlegendentry{Acrobot ($d{=}6$)}
\addplot[cInk, mark=square*, mark size=1.4pt]
  table[x=B, y=lunar, col sep=comma]{data/scaling_fidelity.csv};
\addlegendentry{LunarLander ($d{=}8$)}
\end{axis}
\begin{axis}[
  name=sc2, at={(sc1.south east)}, anchor=south west, xshift=1.3cm,
  width=0.46\textwidth, height=5.0cm,
  xlabel={resolution $B$}, ylabel={clauses (interpretability cost)},
  xmin=1.5, xmax=12.5, xtick={2,4,6,8,12}, ymin=0,
  tick label style={font=\scriptsize}, label style={font=\scriptsize},
  every axis plot/.append style={thick}]
\addplot[cGold, mark=*, mark size=1.5pt]
  table[x=B, y=cartpole, col sep=comma]{data/scaling_clauses.csv};
\addplot[cGoldDeep, mark=triangle*, mark size=1.8pt]
  table[x=B, y=acrobot, col sep=comma]{data/scaling_clauses.csv};
\addplot[cInk, mark=square*, mark size=1.4pt]
  table[x=B, y=lunar, col sep=comma]{data/scaling_clauses.csv};
\end{axis}
\end{tikzpicture}
\caption{The resolution sweep as a test of the conversion theory, DAgger
student, interquartile mean over fifteen seeds for CartPole and Acrobot and five
for LunarLander. Left, held-out fidelity to the
network rises with the resolution on every task, the empirical
$\epsilon^{\dagger}_B \to 0$ of Theorem~\ref{thm:resolution}; the
eight-dimensional LunarLander plateaus below the lower-dimensional tasks, the
fidelity face of the return ceiling. Right, the clause count, the
interpretability price of that fidelity, grows with the resolution, so $B$ is
the exchange rate between how finely the rules partition the state space and how
many of them a reader must hold in mind.}
\label{fig:scaling}
\end{figure}
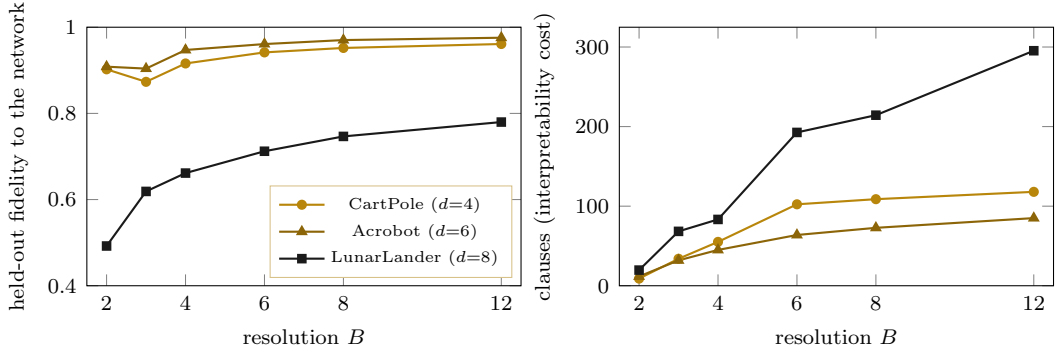

The artefact remains a program one reads, and the interpretability cost is
explicit. On Acrobot the eleven-clause list induced at $B=2$, shown in Listing
\ref{lst:acrobot}, reads as a hand-writable swing-up controller that torques the
inner link in the direction of its angular velocity and torques toward the
raised configuration otherwise, and the $B=2$ DAgger student scores an
interquartile-mean return of $-87.2$ with a success rate of $0.87$, close to the
teacher's $-84.3$ with a rule base a human can audit at a glance. The number of
clauses grows with resolution, from about twelve at $B=2$ to about eighty-five at
$B=12$ on Acrobot and from nine to about a hundred and twenty on CartPole, so the
resolution knob is precisely the exchange rate between how finely the program
partitions the state space and how much of it a reader can hold in mind. CartPole
supplies the cautionary dual: its nine-clause $B=2$ list reads as the textbook
pole balancer, pushing toward the side the pole is tipping, and it agrees with the
network on about $90\%$ of held-out states, yet it scores only about $210$ with a
success rate near zero, because coarse thresholds accumulate error over the long
balancing horizon. Fidelity to the teacher is necessary but not sufficient for
closed-loop return, and the resolution knob is what buys the missing precision,
with $B \ge 3$ clearing the solved line.

\begin{lstlisting}[style=prolog, caption={The Acrobot expert system induced by
DAgger at resolution $B=2$, seed $0$, copied verbatim from the emitted file and
lightly aligned. Eleven clauses over threshold literals on the six observation
features partition the state space into torque regions; the program runs
unchanged under SWI-Prolog, and the $B=2$ DAgger student scores an
interquartile-mean return of $-87.2$ with a $0.87$ success rate over the fifteen
seeds, against the teacher's $-84.3$.}, label={lst:acrobot}]
% features: cos1,sin1,cos2,sin2,ang_vel1,ang_vel2 (link angles + rates)
act(S, torque_pos) :- ang_vel2(S) >= 0.148,  ang_vel1(S) < -0.022, !.
act(S, torque_neg) :- ang_vel1(S) >= -0.022, ang_vel2(S) <  0.148, !.
act(S, torque_neg) :- sin1(S)    <  -0.027,  cos1(S)     <  0.746, !.
act(S, torque_pos) :- cos2(S)    >=  0.416,  sin2(S)     <  0.089, !.
act(S, torque_neg) :- cos1(S)    >=  0.746,  ang_vel1(S) < -0.022, !.
act(S, torque_neg) :- sin2(S)    >=  0.089,  cos1(S)     <  0.746, !.
act(S, torque_pos) :- ang_vel1(S) >= -0.022, sin1(S)    >= -0.027, !.
act(S, torque_zero) :- sin1(S)   >= -0.027, !.
act(S, torque_neg) :- cos2(S)    >=  0.416, !.
act(S, torque_pos) :- !.               % unconditional
act(_S, torque_pos).                   % default
\end{lstlisting}

\section{Discussion and limitations}
\label{sec:discussion}

The results support a reading that turns on representation. On a task small
enough to admit exact evaluation, a trained neural policy is rewritten as a
short, executable, human-readable logic program that plays the task at least as
well, and after a return-driven expansion better, than the network it came from,
with the return loss certified and non-vacuous. Where the observation carries
relational structure, as in DoorKey, the first-order form of that program is what
lets it transfer across scale while a propositional coordinate tree cannot. Where
the observation carries none, as in low-dimensional continuous control, the same
first-order form has nothing to bind and cedes the return-versus-size frontier to
the decision trees built for that setting. The honest summary is that our method
occupies a specific and defensible cell of the design space rather than
dominating everywhere, and the limitations below say where its edges are.

The first limitation is scale. Proposition \ref{prop:certificate}, the
advantage-gap certificate of Proposition \ref{prop:advantage-certificate}, and
the expansion oracle rely on model access and on an enumerable reachable set, and
the census costs a linear solve per evaluation, which is the deliberate price of
exactness at the $10^4$-state scale of KeyDoor. The continuous-control study of
Section \ref{sec:continuous} is what the method becomes beyond that scale: the
census gives way to the DAgger sampling of \citet{ross2011reduction} and the
exact solves to Monte Carlo evaluation with bootstrap intervals, at the cost of
reintroducing confidence parameters into what is a certificate in the finite
case, and the propositional student trades the exact-return guarantee for a
resolution knob whose ceiling is visible on LunarLander.

The second limitation is representational, and the continuous-control study is
where we confront it rather than paper over it. An ordered decision list commits
to one feature at a time and matches the first clause that fires, so it cannot
express the shared, reused splits that a CART or VIPER tree distributes across a
branching structure, and on a four to six dimensional vector of physical
quantities, where no relational regularity rewards the first-order form, that
inability costs return: the list loses to both trees on CartPole in zero of
fifteen seeds and matches rather than beats them on Acrobot. We verified that
this is not an artifact of the data-driven threshold placement by refitting the
cuts at the Gini-optimal split points a tree would use, which closes none of the
gap on average, so the deficit is a property of the hypothesis class and not of
its tuning. We state it plainly because it is the mirror image of the DoorKey
result and shares its cause: the representation that transfers across scale is
the one that gives up joint splitting, and the setting that rewards joint
splitting is the one with no relational structure to transfer.

The third limitation is the predicate vocabulary, which we fixed by hand; the
direction and adjacency predicates of KeyDoor and the navigation predicate of
DoorKey encode task knowledge, and although the induction and expansion stages
choose how to combine them, they cannot invent a predicate the vocabulary lacks,
which is exactly the gap that predicate-invention methods address
\citep{sha2024expil} and that object-centric extractors
\citep{delfosse2024interpretable} would fill automatically. The fourth concerns
the certificate, whose standing the advantage-gap refinement changes but does not
make unconditional. On the key-and-door task the refined bound is exact and
non-vacuous because the distilled and expanded students weakly dominate their
teacher at the states where they disagree, the sign condition of
Corollary \ref{cor:nonvacuity}; where that condition fails, as on five of the
capped-regime distilled seeds, the bound reverts toward the worst-case horizon
factors, and it offers no a-priori guarantee before the model is inspected. Its
contribution is an exactly computable, horizon-aware certificate that is checked
on every run and is tight in the regime the distillation targets, not a bound
that holds for free at any discount factor.

The natural continuation returns to the framing that motivated this work. If
the predicate vocabulary and the candidate rules can be proposed by a
generative model rather than fixed in advance, the extraction and expansion
loop becomes a way to let a large language model draft the symbolic theory of a
neural agent and let exact policy evaluation keep only the drafts that improve
the return, joining the post-hoc distillation studied here to the generative
construction of expert systems \citep{garrido2025gofai}. The guarantees would
remain those of Section \ref{sec:theory}: the acceptance oracle certifies every
edit, whoever proposes it. A proof of concept substantiates this
interchangeability claim on the capped key-and-door regime, where the distilled
programs are furthest from the optimum. We asked a large language model
(Claude, Anthropic), operating as an interactive coding agent with access to
the three distilled rule lists of seeds three, five, and seven, to diagnose
their failure mode and author repair edits, which were then frozen verbatim in
a driver and submitted to the identical exact acceptance oracle used by the
heuristic expansion; the proposer has no influence on acceptance, and replaying
the frozen proposals reproduces the identical accept and reject decisions and
the identical exact returns. The model identified that the distilled
walk-to-key clause carries a door-adjacency guard that strands keyless states
adjacent to the door, authored one unguarded walk-to-key clause and, for the
third seed, one on-key pickup clause, and every one of the four authored
proposals was accepted, driving all three seeds from exact returns of
$0.4977$, $0.4977$, and $0.3435$ to the exact optimum $0.812204$ in one, one,
and two accepted edits respectively. The contrast with the heuristic search is
the point: the random first-improvement expansion reached the same optimum on
these seeds only after $65$, $98$, and $57$ exact Bellman solves of candidate
evaluation, whereas the informed proposer required one oracle call per
accepted edit, so the certificate machinery converts a generative model from
an unaccountable oracle into a proposal engine whose every suggestion is
exactly audited before it enters the program. Two qualifications keep this
honest. The proposals were authored by an interactive agent session rather
than a fixed scripted pipeline, so proposal authorship is not reproducible
from a random seed even though the certified trajectory is fully replayable
from the frozen queue, and the study is a three-seed proof of concept on the
enumerable regime, offered as evidence that the acceptance oracle makes the
proposer interchangeable rather than as a study of generative rule synthesis
at scale.

\section{Conclusion}

We introduced a post-hoc transformation that distils a frozen deep
reinforcement learning policy into an executable first-order Prolog program and
then expands the program by exact-return hill climbing until it matches or
surpasses its teacher. We proved a return-loss bound, sharpened it into an
advantage-gap certificate that is horizon-aware, exactly computable, and
non-vacuous in the regime the distillation targets, and proved that the expansion
loop improves monotonically and halts. On a two-room key-and-door task the
expanded program reaches exact optimal return in every seed and exceeds the
stochastic teacher on exact return in fifteen of fifteen capped-regime seeds,
while the refined certificate is tighter than the worst-case bound by a median
factor of about $13{,}700$ and non-vacuous throughout. On a MiniGrid DoorKey task
the first-order student distilled on one grid size transfers to unseen sizes with
nine clauses, where decision trees over absolute coordinates collapse to zero
return, so the relational representation is what carries the policy across scale.
On propositional continuous control the same pipeline substitutes the network in
fidelity and beats the closest logic-based prior art but loses on return to the
decision-tree distillers, the expected price of a representation with no
relational structure to exploit. The symbolic student is not a faithful clone of
the network but a repaired program, and the mechanism that lets it surpass the
teacher, the certification of each edit by exact policy evaluation, is the same
mechanism that would let a generative proposer of rules improve a neural agent
under a guarantee.

\appendix

\section{Reproducibility}
\label{app:repro}

Every number in the paper is produced by a pinned pipeline, and we record the
provenance here so that the study can be rerun without guesswork. All experiments
run on CPU under Python $3.10.12$ with NumPy $1.26.4$, SciPy $1.15.3$, PyTorch
$2.9.1$ built for CPU, Gymnasium $1.3.0$, and SWI-Prolog $9.2.9$; every emitted
Prolog program is executed by that SWI-Prolog build and asserted to select the
same action as the reference evaluator, on all reachable states in the
key-and-door task and on at least five hundred sampled states in the DoorKey and
continuous-control tasks, so the artefact analysed and the artefact executed are
one program. The teacher in every task is a proximal policy optimization agent
\citep{schulman2017proximal} with a two-layer tanh actor-critic and eight
synchronous environments, trained with generalized advantage estimation, and the
full learning-telemetry panel, explained variance, approximate Kullback-Leibler
divergence, entropy, clip fraction, and per-group gradient norms, is written to
disk per update and audited before any policy is distilled; the complete
hyperparameter set and the per-seed random streams are pinned in the released
configuration alongside the SHA-256 hash of each environment's source file, and
the audit confirms that the teachers are healthy, the low explained variance on
the fast-solving CartPole seeds being the known constant-value-target artifact
rather than a pathology.

The key-and-door environment enumerates to $16{,}944$ reachable states, uses a
discount factor $\gamma = 0.99$, a horizon of $120$, and a maximum absolute
one-step reward $R_{\max} = 0.99$, and is studied over fifteen seeds in each of
two teacher regimes: the converged regime trains to at most $800$k environment
steps with an early stop at greedy success $0.95$, reached between $369$k and
$451$k steps, and the capped regime fixes a budget of $299{,}008$ steps.
Aggregation over the fifteen seeds is by an interquartile mean with a $10{,}000$
resample stratified bootstrap in the manner of \citet{agarwal2021deep}, and
paired comparisons use the two-sided exact Wilcoxon signed-rank test with a
Bonferroni correction within each six-test regime family. The DoorKey study
trains one teacher per seed on the $8 \times 8$ layout over fifteen seeds and
evaluates every distilled student on the $6 \times 6$, $8 \times 8$, and
$16 \times 16$ layouts without retraining, with a Holm-Bonferroni correction over
the eight-contrast relational-versus-absolute family. The continuous-control
study sweeps the resolution ladder $B \in \{2, 3, 4, 6, 8, 12\}$ over fifteen
seeds each on CartPole and Acrobot and five seeds on LunarLander, evaluates every
student on three hundred paired Monte-Carlo episodes with a fixed episode-seed
stream shared by the teacher and all students, and corrects the per-environment
baseline comparison family by Holm-Bonferroni. The whole suite fits within a
compute budget of the order of fifty CPU-hours on a sixteen-core machine with no
accelerator; the key-and-door runs, for instance, average about one minute each,
of which the expansion hill climb is under a second. The generative proof of
concept of Section \ref{sec:discussion} is the one measurement whose proposal
authorship is not seeded: the language-model-authored edits are frozen verbatim
in the released driver, so re-running it replays the identical certified
trajectory through the same exact oracle, and only the authorship of the frozen
queue itself is not derivable from a random seed.

\section{Baseline protocol}
\label{app:baselines}

The interpretable baselines are reimplemented to share our extraction and
evaluation harness, and three approximations are worth stating so that the
comparison is read for what it is. First, the VIPER baseline
\citep{bastani2018verifiable} is adapted to a policy-gradient teacher. The
original importance weight $\ell(s) = V(s) - \min_a Q(s, a)$ requires an
action-value critic, which our proximal policy optimization teacher does not
expose, so we substitute the margin between the top two action logits as an
advantage-magnitude proxy and resample the DAgger training states by that weight;
the resampling is the only difference from a DAgger-trained tree, and we label
the method VIPER on that understanding rather than as a literal reproduction.
Second, the Coppens baseline \citep{coppens2021rule} is faithful in spirit rather
than in code: a CN2-style greedy inducer grows an unordered set of rules over the
same threshold predicate vocabulary our decision list uses, and it classifies a
state by majority vote among the rules that fire, falling back to a default
action on ties or when no rule fires, so the only representational difference from
our method is the unordered set-valued form against our ordered first-order list
with a default. Third, the CART baseline \citep{breiman1984classification} splits
the raw continuous observation with its own learned thresholds rather than our
fixed resolution-$B$ quantile grid, and we report it both unconstrained and
size-matched to the decision list's clause count, so the reader can separate the
effect of representation from the effect of size. On DoorKey the same CART and
VIPER inducers are fitted twice, once over the agent-relative relational features
the Prolog student reads and once over the absolute cell coordinates of the
objects and the agent, which is the contrast that isolates representation from
induction algorithm. Continuous-control baselines are confined to CartPole and
Acrobot, which share the environment-agnostic harness; a fair LunarLander
comparison would require porting its separate interface to that harness and is
left to future work.

\bibliographystyle{plainnat}
\bibliography{refs}

\end{document}